\documentclass{article}
\usepackage{harp_preprint,times}

\usepackage{latexsym}
\usepackage[T1]{fontenc}
\usepackage[utf8]{inputenc}
\usepackage{newunicodechar}
\newunicodechar{–}{--}
\newunicodechar{—}{---}
\usepackage{microtype}

\usepackage{hyperref}
\hypersetup{
  hidelinks,
  pdftitle={HARP: Efficient Data Selection for Finetuning Large Language Models},
  pdfauthor={Ning Wang, Zhengxin Zhang, Maosen Tang, Yitang Gao, Claire Cardie, Sainyam Galhotra}
}
\usepackage{url}
\usepackage{booktabs}
\usepackage{multirow}
\usepackage{amsfonts}
\usepackage{nicefrac}
\usepackage{xcolor}
\usepackage{graphicx}
\usepackage{wrapfig}
\usepackage{float}
\usepackage{longtable}
\usepackage{amssymb}
\usepackage{amsthm}
\usepackage{amsmath}

\newtheorem{lemma}{Lemma}
\newtheorem{proposition}{Proposition}
\newtheorem{theorem}{Theorem}

\title{HARP: Efficient Data Selection for Finetuning Large Language Models}

\author{
\textbf{Ning Wang}\textsuperscript{1}\thanks{Equal contribution; author order determined by coin flip.}
\quad
\textbf{Zhengxin Zhang}\textsuperscript{1}\footnotemark[1]
\quad
\textbf{Maosen Tang}\textsuperscript{1}
\quad
\\\textbf{Yitang Gao}\textsuperscript{2}
\quad
\textbf{Claire Cardie}\textsuperscript{1}
\quad
\textbf{Sainyam Galhotra}\textsuperscript{1}
\\[0.5em]
\textsuperscript{1}Cornell University
\quad
\textsuperscript{2}The Hong Kong University of Science and Technology
\\[0.5em]
\texttt{\{nw366, zz865\}@cornell.edu}
}

\begin{document}
\maketitle

\begin{abstract}
Finetuning data selection requires balancing two competing goals:
selecting examples that improve the downstream objective, and doing so
without repeatedly finetuning models. Train-free selectors are scalable but rely on proxies such as
embedding similarity or clustering, which may not match the target
objective. Train-based selectors better reflect downstream utility through
gradient signals, subset evaluation, or Shapley attribution, but
require many costly train--evaluate iterations. We propose
\textbf{H}ierarchical \textbf{A}ctive \textbf{R}egion
\textbf{P}runing (\textbf{HARP}), an efficient train-based selector that preserves
downstream alignment while reducing selection cost. HARP organizes the
training pool into a node--leaf hierarchy, evaluates only representative
leaves, and infers unmeasured utilities with empirical Bayes posteriors.
It then selects data using two complementary envelopes:
\textbf{HARP-C}, which conservatively controls redundancy, and
\textbf{HARP-E}, which additively rewards complementary regions.
We theoretically show that, under local smoothness and bounded estimation error, HARP controls selection error while reducing train--evaluate cost. Across $3$ base models, $3$ finetuning datasets and $4$ reasoning benchmarks ($36$
settings), HARP variants take the per-cell best among ranked methods on $23$ cells and achieve
$36$-setting means of $64.2$ / $63.5$, compared with $59.3$ for the strongest
baseline, \textsc{NICE}. \textsc{HARP-C} trains on roughly $9\times$ fewer
examples than the budgeted baselines and $75\times$ fewer than full finetuning,
combining higher mean accuracy with a smaller mean training-set size than every
ranked baseline.
\end{abstract}

\section{Introduction}
Finetuning \citep{zhang2024quantized,chung2024scaling,longpre2023flan,wang-etal-2023-self-instruct,dettmers2023qlora,hu2022lora} has become an effective approach for specializing foundation large language models (LLMs) to follow task specifications, improve downstream performance, and adapt to domain-specific needs.
Recent studies \citep{zhou2023lima} show that a relatively small subset of high-quality data can match or even surpass the performance achieved using the full dataset. This suggests that current finetuning datasets are often noisy and inefficient, containing redundancy, near-duplicate instances, low-signal examples, and distributional mismatches with the target evaluation setting.
Nevertheless, building high-quality finetuning datasets requires extensive expert labor, making it impractical at scale.
Recent works are primarily based on train-free data selection and train-heavy evaluation-based data selection. Train-free selection methods \citep{ge-etal-2024-clustering,tsds,yang-etal-2025-measuring,dq,dsir} typically use heuristics or representations (e.g., embeddings, clustering, diversity/coverage objectives, or quality proxies) to choose subsets without repeatedly finetuning candidate subsets. These methods are cheap and can be applied before training. However, their proxies can be misaligned with the actual downstream objective.
For example, some methods construct a training subset by ranking training examples by embedding similarity to target validation examples. When the downstream objective is task accuracy, high embedding similarity fails to guarantee high marginal utility for improving the downstream objective.
Train-based data evaluation explicitly estimates the contribution of training examples (or groups of examples) to a downstream objective via influence/gradient-based approximations \cite{less,nice}, or Shapley-style marginal contribution estimates \cite{shed}.
While objective-faithful, these approaches require substantial computational cost, driven by the finetuning and full-set evaluation costs per cycle and the total number of train--evaluate iterations. This motivates a
middle ground: a selector that retains the downstream alignment of
train-based evaluation, but avoids exhaustively evaluating many candidates.
In this paper, we propose \textbf{H}ierarchical \textbf{A}ctive \textbf{R}egion \textbf{P}runing (\textbf{HARP}), an efficient training-based data selection framework for LLM finetuning. HARP organizes the training set into a node--leaf hierarchy, where each leaf contains a group of examples, and each node contains a group of related leaves. Instead of evaluating every leaf, HARP estimates the downstream utility of a small set of representative leaves and infers the utility of unmeasured leaves through empirical Bayes posteriors. Building on these estimates, we introduce two complementary selection envelopes: \textbf{HARP-C} and \textbf{HARP-E}. HARP-C uses a \textbf{C}onservative coverage envelope to avoid over-counting utility under redundancy, while HARP-E uses an \textbf{E}xpansive additive envelope to credit multiple complementary leaves within the same domain. We theoretically show that, under local smoothness and bounded estimation error, HARP provides stable utility estimates and controls selection error while substantially reducing the train-evaluate iterations.

We evaluate HARP across $3$ base models, $3$ fine-tuning
datasets, $8$ baselines and $4$ reasoning benchmarks ($36$ settings).
HARP variants achieve the best result among ranked methods in $\mathbf{23}$ settings ($64\%$).
\textsc{HARP-C} and \textsc{HARP-E} reach mean accuracies of $64.2$ and
$63.5$, outperforming the strongest baseline, \textsc{NICE}, by $+4.9$ and
$+4.2$ points, respectively. \textsc{HARP-C} uses roughly $9\times$ fewer
training examples than the $10$k-budget baselines and $75\times$ fewer than
full finetuning, highlighting the effectiveness of efficient data selection for finetuning large language models.

\section{Related Work}
\noindent\textbf{Train-free data selection methods for finetuning.}
In contrast, \emph{train-free} methods select subsets using inexpensive \emph{proxy signals} that avoid repeatedly finetuning candidate subsets. A common strategy is \emph{geometry-driven selection}, where methods use embeddings together with clustering, submodular objectives, or diversity sampling to retain representative and non-redundant examples (e.g., \cite{ge-etal-2024-clustering,tsds,bukharin2024data,yang-etal-2025-measuring}). Another strategy uses \emph{quality proxies}, such as heuristic filters, lightweight model-based scoring, or proxy difficulty/utility estimates, to remove noisy or low-value examples before training; recent LLM-specific instances include self-guided difficulty-based selection, one-shot utility scoring, instruction mining, automatic quality/complexity scoring, and weak-to-strong filtering (e.g., \cite{li2024quantity,li2024one,cao2023instruction,liu2023makes,li2024superfiltering}). In the adjacent setting of retrieval-augmented generation, \citet{deng2025influence} select inference-time contexts using predicted contextual influence. Other data curation methods synthesize finetuning data through text-grounded task design (e.g., \cite{chen2023tegit}). These methods scale well to large finetuning corpora, but their proxy objectives can be misaligned with the downstream metric and generally fail to capture \emph{interaction effects} that emerge only under actual training dynamics.

\noindent\textbf{Train-based (objective-aligned) data selection methods for finetuning.}
A complementary line of work studies \emph{objective-aligned} data selection, where the value of a training example (or group) is defined by its estimated effect on downstream performance after finetuning.
Earlier training-based approaches explore several forms of model-in-the-loop selection, including transferred Shapley-value-based data valuation \citep{schoch2023data}, iterative self-evolving subset refinement \citep{wu2023self}, and active task selection based on prompt sensitivity \citep{kung2023active}.
\citet{shed} further moves closer to the true ``fine-tune then evaluate'' objective by using a Shapley-style refinement framework for instruction finetuning, but this comes with substantial computational overhead.
To reduce the cost of repeated finetuning, \citet{less} proposes an influence-based approximation that builds a reusable gradient datastore and selects data using optimizer-aware gradient similarity to a small target set, enabling targeted capability induction with only a small fraction of the data.
More recently, \citet{nice} (NICE) extends this direction to \emph{non-differentiable} evaluation metrics via policy-gradient-style influence estimation, bringing data selection closer to the true task objective when loss-based proxies are misaligned.

\noindent\textbf{Finetuning of large language models.}
Supervised finetuning (SFT) adapts pretrained LLMs to follow natural-language instructions using instruction--response data, and has become a standard post-training stage for improving instruction following and generalization. Prior work shows that both scale and mixture design matter: large instruction collections and carefully balanced task mixtures can substantially improve transfer, while parameter-efficient methods reduce the cost of adaptation \citep{zhang2024quantized,chung2024scaling,longpre2023flan,hu2022lora,dettmers2023qlora,zhang2023llama,xu2023baize}. At the data-construction layer, synthetic pipelines bootstrap supervision from seed tasks, teacher models, or reverse/task-grounded generation, including Self-Instruct, Stanford Alpaca, GPT-4-based instruction generation, and LongForm, while AlpaGasus filters existing instruction data \citep{wang-etal-2023-self-instruct,alpaca,peng2023instruction,koksal2024longform,chen2023alpagasus}. Adjacent post-training work further studies preference-based optimization and scaled feedback data, while domain-specific tuning pipelines reinforce that supervision design and data composition are often as important as raw scale \cite{dubois2023alpacafarm,rafailov2023direct,cui2023ultrafeedback,yue2023mammoth,yu2023metamath,luo2023wizardmath}. Conversely, LIMA shows that a small but carefully curated set can still be highly effective, underscoring SFT's sensitivity to redundancy and example quality \cite{zhou2023lima}.

\section{Method}
\label{sec:method}

\begin{figure}
    \centering
    \includegraphics[width=\linewidth]{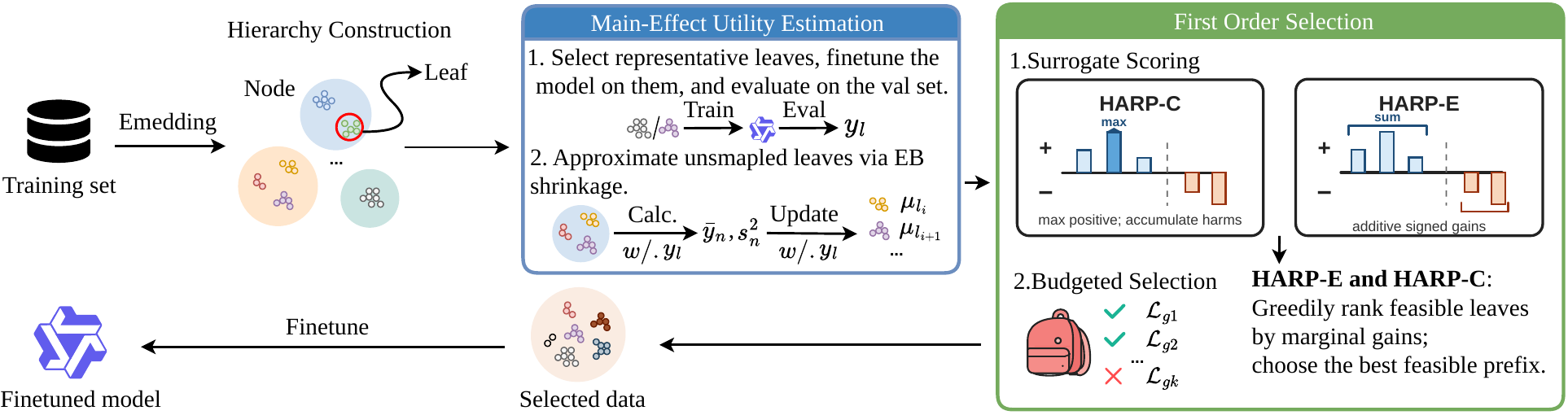}
    \caption{An overview of HARP.}
    \label{fig:framework}
\end{figure}

\subsection{Problem Setup and Overview}
\label{subsec:problem_setup}
Let
$\mathcal{D}=\{x_i\}_{i=1}^N$ denote the training dataset. Our goal is to select a subset
of training examples with total size at most $B$ that maximizes downstream
performance. Let $U(A)$ denote the task-averaged downstream utility (e.g., accuracy) after
fine-tuning on $A\subseteq\mathcal{D}$. The ideal objective is
\(
\max_{A\subseteq\mathcal{D},\, |A|\le B} U(A).
\)
However, this ideal objective is not directly tractable in practice.
Evaluating a candidate subset requires a train--evaluate iteration, and the number of budget-feasible subsets grows
combinatorially with the size of the training data. HARP makes this problem
tractable. First, it
organizes the training data into a hierarchy and selects leaves rather than
individual examples, reducing the number of candidate units that must be
measured. Once the hierarchy below defines leaves \(\{\mathcal{L}_g\}_{g=1}^{G}\), each selected leaf set \(S\subseteq[G]\) corresponds to \(A(S)=\bigcup_{g\in S}\mathcal{L}_g\), with cost \(c(S)=\sum_{g\in S}|\mathcal{L}_g|\), and we write \(U(S)=U(A(S))\). HARP optimizes the leaf-restricted target \(\max_{S\subseteq[G],\,c(S)\le B}U(S)\), which is a restriction of the example-level objective above. This leaf-level representation lets us score the leaf-restricted budgeted objective with a first-order utility approximation: HARP estimates
the contribution of each leaf and then aggregates these main
effects to score candidate selections. Second, we introduce two complementary
domain-wise envelopes. \textbf{HARP-C} uses a conservative coverage envelope,
crediting only the strongest positive leaf in each domain while accumulating all
estimated harms; this limits over-counting when several leaves capture the same domain. \textbf{HARP-E} uses an expansive additive envelope, summing positive
contributions across leaves so that multiple complementary leaves can receive
credit within the same domain.

Let \(\mathcal{E}=\bigcup_{d=1}^{D}\mathcal{E}_d\) be the fixed full
validation set used for selection, partitioned into nonempty domains, and let
\(n_d=|\mathcal{E}_d|\). Every measured model is scored on all examples in
each \(\mathcal{E}_d\). For a model \(f_A\) fine-tuned on training subset \(A\),
the domain utility and base-model reference are
\[
u_d(A)=\frac{1}{n_d}\sum_{e\in\mathcal{E}_d}m(f_A(e),y_e),
\qquad
v_0(d)=\frac{1}{n_d}\sum_{e\in\mathcal{E}_d}m(f_0(e),y_e),
\]
where \(f_0\) is the base model and the normalized metric \(m\) lies in
\([0,1]\). The domain partition and evaluation examples remain fixed
throughout selection.

\subsection{Hierarchical Node--Leaf Construction}
\label{subsec:hierarchy}

We further reduce the number of train--evaluate iterations by treating leaves as the atomic selection units instead of individual examples. Let \(\hat z_i\in\mathbb{R}^p\) denote the normalized embedding of training example \(x_i\). HARP builds a two-level hierarchy over the training pool: node groups provide coarse local regions for sharing utility estimates, and leaves \(\{\mathcal{L}_g\}_{g=1}^{G}\) define the atomic units used for selection. For any subset \(Q\subseteq\mathcal{D}\) \ and a specified target number \(M_Q\) of child groups, we use a coverage-oriented anchor partitioning procedure. The first anchor is chosen as the example closest to the centroid of \(Q\). Subsequent anchors are selected by the farthest-first rule \(a_r=\arg\max_{x_i\in Q}\min_{a\in A_{r-1}}(1-\hat z_i^\top \hat z_a)\), where \(A_{r-1}\) is the set of previously selected anchors. Each example in \(Q\) is then assigned to its closest anchor by cosine similarity. This procedure is first used to form node groups and then applied within each node to produce leaves, recursively splitting groups that exceed the \ maximum leaf size \(C_{\max}\) and merging leaves with fewer than \(C_{\min}\) examples into the nearest valid sibling by centroid similarity. The resulting leaves cover diverse regions of the training pool while keeping each selection unit large enough for stable utility estimation. The minimum size is a construction target; the bound below assumes that the final leaves satisfy it. As an immediate upper bound, replacing individual examples with leaves reduces exhaustive utility estimation from \(N\) individual-example queries to at most \(G\) leaf-level queries, with \(G\le N/C_{\min}\) whenever the final leaves satisfy the minimum-size constraint. Appendix~\ref{app:hierarchy_cost} formalizes this iteration-reduction bound.

\subsection{Main-Effect Estimation}
\label{subsec:main_effect_estimation}

We then estimate the utility of each leaf. Let \(p(g)\) denote the parent node of leaf \(g\), and let \(\mathcal{P}_p\) denote the leaves under node \(p\). Each leaf is represented by its mean embedding \(\bar z_g=|\mathcal{L}_g|^{-1}\sum_{x_i\in\mathcal{L}_g}\hat z_i\). For domain \(d\), define the centered main effect \(\phi_g(d)=u_d(\mathcal{L}_g)-v_0(d)\), where \(u_d(\mathcal{L}_g)\) is the domain utility after fine-tuning on \(\mathcal{L}_g\), and \(v_0(d)\) is the base-model domain utility. Since evaluating all \(G\) leaves is expensive, we measure only representative leaves \(\mathcal{R}_p\subseteq\mathcal{P}_p\)\ with \(|\mathcal{R}_p|\ge 1\) for every nonempty parent node \(p\), selected by size stratification and farthest-first sampling over \(\bar z_g\). For each measured representative \(r\), we observe \(y_r(d)=u_d(\mathcal{L}_r)-v_0(d)\), the noisy measured main effect. For an unmeasured leaf \(g\), we first form a local estimate \(\widetilde y_g(d)\) from measured representatives in the same node: \(\widetilde y_g(d)=\sum_{r\in\mathcal{R}_{p(g)}}\alpha_{g,r}y_r(d)\). Here \(\alpha_{g,r}\) is the similarity weight assigned from representative \(r\) to leaf \(g\), defined by \(\alpha_{g,r}=\kappa(\bar z_g,\bar z_r)/\sum_{r'\in\mathcal{R}_{p(g)}}\kappa(\bar z_g,\bar z_{r'})\), where \(\kappa(\bar z_g,\bar z_r)=\exp(((\bar z_g/\|\bar z_g\|_2)^\top(\bar z_r/\|\bar z_r\|_2))/\lambda)\) is a cosine-similarity kernel computed from normalized leaf mean embeddings and \(\lambda>0\) controls locality. The interpolation step gives a local representative-based estimate of \(\phi_g(d)\), with its approximation error controlled separately by the representative interpolation bound. We then apply empirical-Bayes shrinkage: \(\widehat{\phi}_g(d)=\rho_g(d)\widetilde y_g(d)+(1-\rho_g(d))\mu_0(d)\). Here \(\widehat{\phi}_g(d)\) is the final estimated main effect, \(\mu_0(d)\) is the global measured-effect mean, and \(\rho_g(d)=\tau^2(d)/(\tau^2(d)+\sigma_{p(g)}^2(d)/n_{\mathrm{eff}}(g))\) is the shrinkage weight. The terms \(\tau^2(d)\) and \(\sigma_{p(g)}^2(d)\) are the across-node and within-node variances, respectively, and \(n_{\mathrm{eff}}(g)=1/\sum_{r\in\mathcal{R}_{p(g)}}\alpha_{g,r}^2\) is the effective number of representatives supporting leaf \(g\). For directly measured leaves, we set \(\widehat{\phi}_g(d)=y_g(d)\). This yields the domain-wise main-effect matrix \(\widehat{\Phi}=[\widehat{\phi}_g(d)]_{g\in[G],d\in[D]}\). \ We keep active domains \(\mathcal{A}_{\mathrm{dom}}=\{d:\max_g|\widehat{\phi}_g(d)|>\epsilon_{\mathrm{dom}}\}\). If \(\mathcal{A}_{\mathrm{dom}}=\varnothing\), we set \(\mathcal{A}_{\mathrm{dom}}=[D]\). We then normalize their weights \(w_d\), write \(b_d=v_0(d)\), and fix \(\mathcal{A}_{\mathrm{dom}}\), \(w_d\), and \(b_d\) before defining both true and estimated surrogate objectives. This keeps domains with either estimated gains or estimated harms active, which is necessary because the conservative objective below accumulates negative effects rather than ignoring them. We justify the interpolation error, shrinkage estimator, and train--evaluate iteration reduction in Appendices~\ref{app:rep_interpolation}, \ref{app:eb_shrinkage}, and \ref{app:main_effect_query_reduction}.

\subsection{HARP-C: Conservative Coverage Envelope}
\label{subsec:harp_fo_cons}

An additive first-order objective can over-count leaves: if multiple leaves improve the same domain, summing their positive effects may predict unrealistically large gains even when domain performance saturates. We therefore use a conservative first-order envelope. For each leaf \(g\) and active domain \(d\), write \(\widehat{\phi}_g^+(d)=\max(\widehat{\phi}_g(d),0)\) and \(\widehat{\phi}_g^-(d)=\max(-\widehat{\phi}_g(d),0)\). Given a selected leaf set \(S\), we predict the conservative domain utility as \(\widehat y_d^{\mathrm{C}}(S)=\mathrm{clip}_{[0,1]}(b_d+\max_{g\in S}\widehat{\phi}_g^+(d)-\sum_{g\in S}\widehat{\phi}_g^-(d))\), with the convention \(\max_{g\in\varnothing}\widehat{\phi}_g^+(d)=0\), and define \(\widehat U_{\mathrm{C}}(S)=\sum_{d\in\mathcal{A}_{\mathrm{dom}}}w_d\widehat y_d^{\mathrm{C}}(S)\). This envelope credits only the strongest positive leaf in each domain while accumulating all estimated harms. It therefore favors leaves whose gains remain useful under a conservative view of within-domain redundancy. HARP-C ranks leaves greedily by marginal gain under \(\widehat U_{\mathrm{C}}\). Let \(c_g=|\mathcal{L}_g|\) be the cost of leaf \(g\). Starting from \(S_0=\varnothing\), at step \(k\) \ we define the feasible candidate set
\(
\mathcal{F}_k
=
\{g\notin S_k:\ c(S_k)+c_g\le B\}.
\)
If \(\mathcal{F}_k=\varnothing\), the greedy construction stops. Otherwise, we choose
\(
g_{k+1}
=
\operatorname*{arg\,max}_{g\in\mathcal{F}_k}
\left[
\widehat U_{\mathrm{C}}(S_k\cup\{g\})
-
\widehat U_{\mathrm{C}}(S_k)
\right],
\)
and set \(S_{k+1}=S_k\cup\{g_{k+1}\}\), producing a ranked sequence of feasible leaf prefixes. The final HARP-C selection is the best feasible prefix,
\(
\widehat S_{\mathrm{C}}
=
S_{k^\star},
\quad
k^\star
=
\operatorname*{arg\,max}_{k:\,c(S_k)\le B}
\widehat U_{\mathrm{C}}(S_k),
\)
where \(c(S_k)=\sum_{g\in S_k}c_g\). The budget value \(B\) is set in the experimental protocol.

\begin{theorem}[First-order stability of HARP-C]
\label{thm:harp_c_stability}
Let \(U_{\mathrm{C}}\) be the conservative coverage first-order utility defined
with the true main effects \(\phi_g(d)\), and let \(\widehat U_{\mathrm{C}}\)
be its estimated version defined with \(\widehat{\phi}_g(d)\). Let
\(S^\star\in\arg\max_{c(S)\le B}U(S)\), and let \(\widehat S_{\mathrm{C}}\) be the set returned by HARP-C.
Suppose that, for every feasible set \(S\),
\(|U(S)-U_{\mathrm{C}}(S)|\le \epsilon_{\mathrm{C}}\) and
\(|\widehat U_{\mathrm{C}}(S)-U_{\mathrm{C}}(S)|\le \delta_{\mathrm{C}}\).
If the greedy procedure returns an \(\epsilon_{\mathrm{opt}}\)-approximate
maximizer of \(\widehat U_{\mathrm{C}}\) under the budget, then
\(U(S^\star)-U(\widehat S_{\mathrm{C}})
\le 2\epsilon_{\mathrm{C}}+2\delta_{\mathrm{C}}+\epsilon_{\mathrm{opt}}\).
Moreover, if every feasible set contains at most \(K\) leaves, the active-domain set is fixed, the active-domain
weights sum to one, \(b_d\) and \(w_d\) are fixed, and
\(|\widehat{\phi}_g(d)-\phi_g(d)|\le \eta\) for all \(g,d\), then
\(\delta_{\mathrm{C}}\le (K+1)\eta\).
\end{theorem}

Theorem~\ref{thm:harp_c_stability} follows from the shared first-order
stability template in Theorem~\ref{thm:shared_fo_stability} and the HARP-C
sensitivity bound in Lemma~\ref{lem:fo_envelope_sensitivity}. The result shows
that HARP-C is stable when the conservative coverage envelope is a good
surrogate for the true subset utility, the estimated main effects are accurate,
and the greedy optimization gap is small.

\subsection{HARP-E: Expansive Additive Envelope}
\label{subsec:harp_fo_opt}

HARP-C is conservative because each domain can receive positive credit from a single selected leaf. This protects against redundant gains, but it can under-credit genuinely complementary leaves. For example, different leaves may improve the same domain, in which case taking only the maximum positive effect discards the useful additional signal. HARP-E addresses this case with an expansive additive envelope that sums positive contributions while still accumulating estimated harms. For each leaf \(g\) and active domain \(d\), the predicted expansive domain utility for a selected leaf set \(S\) is \(\widehat y_d^{\mathrm{E}}(S)=\mathrm{clip}_{[0,1]}(b_d+\sum_{g\in S}\widehat{\phi}_g(d))\), which is equivalent to summing positive effects and subtracting estimated harms. The aggregated objective is \(\widehat U_{\mathrm{E}}(S)=\sum_{d\in\mathcal{A}_{\mathrm{dom}}}w_d\widehat y_d^{\mathrm{E}}(S)\). HARP-E uses the same feasible set, greedy update, and best-prefix rule as
HARP-C, replacing \(\widehat U_{\mathrm{C}}\) with
\(\widehat U_{\mathrm{E}}\). The two envelopes are complementary: HARP-C limits over-counting under redundancy, while HARP-E can recover multiple positive contributions when leaves within the same domain provide complementary information.

\begin{theorem}[First-order stability of HARP-E]
\label{thm:harp_e_stability}
Let \(U_{\mathrm{E}}\) be the expansive additive first-order utility defined
with the true main effects \(\phi_g(d)\), and let \(\widehat U_{\mathrm{E}}\)
be its estimated version defined with \(\widehat{\phi}_g(d)\). Let
\(S^\star\in\arg\max_{c(S)\le B}U(S)\), and let \(\widehat S_{\mathrm{E}}\) be the set returned by HARP-E.
Suppose that, for every feasible set \(S\),
\(|U(S)-U_{\mathrm{E}}(S)|\le \epsilon_{\mathrm{E}}\) and
\(|\widehat U_{\mathrm{E}}(S)-U_{\mathrm{E}}(S)|\le \delta_{\mathrm{E}}\).
If the greedy procedure returns an \(\epsilon_{\mathrm{opt}}\)-approximate
maximizer of \(\widehat U_{\mathrm{E}}\) under the budget, then
\(U(S^\star)-U(\widehat S_{\mathrm{E}})
\le 2\epsilon_{\mathrm{E}}+2\delta_{\mathrm{E}}+\epsilon_{\mathrm{opt}}\).
Moreover, if every feasible set contains at most \(K\) leaves, the active-domain set is fixed, the active-domain
weights sum to one, \(b_d\) and \(w_d\) are fixed, and
\(|\widehat{\phi}_g(d)-\phi_g(d)|\le \eta\) for all \(g,d\), then
\(\delta_{\mathrm{E}}\le {K\eta}\).
\end{theorem}

Theorem~\ref{thm:harp_e_stability} follows from the same shared stability
template as HARP-C, using the HARP-E sensitivity bound in
Lemma~\ref{lem:fo_envelope_sensitivity}. Since HARP-E is additive in the signed main effects, \(\widehat{\phi}_g^+(d)-\widehat{\phi}_g^-(d)=\widehat{\phi}_g(d)\), its sensitivity to main-effect estimation error is bounded by \(K\eta\) under the assumptions of the theorem.

\subsection{Computational Complexity}
\label{subsec:method_complexity}

Let \(Q=\sum_p|\mathcal{R}_p|\) be the total number of measured
representative leaves, \(T_L(x,e)\) the cost of finetuning on \(x\) training
examples for \(e\) epochs, and \(T_{\mathrm{eval}}(|\mathcal{E}|)\) the cost of scoring one
model on the full validation set. With one epoch per measured leaf and three
epochs for final finetuning in our experiments, main-effect estimation costs
\(
\sum_p\sum_{r\in\mathcal{R}_p}T_L(|\mathcal{L}_r|,1)
+(Q+1)T_{\mathrm{eval}}(|\mathcal{E}|),
\)
including one base-model evaluation. Each representative leaf contains at
most \(L=C_{\max}\) examples, so the training term is at most \(Q T_L(L,1)\).
Final finetuning costs \(T_L(c(\widehat S),3)\), where
\(c(\widehat S)=|A(\widehat S)|\) counts selected training examples.
HARP reduces the number of train--evaluate iterations; each measured
iteration still evaluates the full set. Hierarchy construction is a one-time shared preprocessing cost, treated as a constant in the train--evaluate analysis. Details and baseline comparisons
are in Appendix~\ref{app:complexity}.

\section{Experimental Setup}
\label{sec:experiments}

\textbf{Models.}
We evaluate three pre-trained base models: \textsc{Llama-3.1-8B-Base}~\cite{llama3},
\textsc{Qwen3-4B-Base} and \textsc{Qwen3-8B-Base}~\cite{qwen3}. All three are
fine-tuned with LoRA; LoRA configuration, optimizer, and batch settings shared across HARP, and the baselines are in Appendix Table~\ref{tab:hp_shared}.

\noindent\textbf{Training datasets.}
Three SFT sources spanning data quality and skill mix:
\textsc{Self-Instruct}~\cite{wang-etal-2023-self-instruct} (SI; $\sim$82k examples, noisy auto-generated),
\textsc{WizardLM/Evol-Instruct-70k}~\cite{xu2024wizardlm} (Wiz; 70k, clean), and
\textsc{Tulu-3-SFT-mixture}~\cite{tulu3} (Tulu; 100k subset, mixed domains).

\noindent\textbf{Validation datasets.} We use the official MMLU ($1{,}531$ examples) and ARC-Challenge ($299$) validation splits, and hold out $500$ training examples each from GSM8K and math12k for the GSM8K and MATH tasks, respectively.

\noindent\textbf{Evaluation datasets.}
Final evaluation uses four reasoning benchmarks: \textsc{MMLU}~\cite{mmlu} (57 subjects,
$14{,}042$ examples), \textsc{ARC-Challenge}~\cite{arc} ($1{,}172$), \textsc{GSM8K}~\cite{gsm8k}
($1{,}319$), and \textsc{MATH-500}~\cite{math_dataset} ($500$). MMLU and ARC are scored by
the first-token log-probabilities of answer-option letters;
GSM8K and MATH use answer-string match on greedy vLLM generations.

\noindent\textbf{Baselines.} 
0) Base: the untuned model without finetuning.
1) Random (Rand): uniform random selection of 10k examples.
2) DSIR~\cite{dsir}: importance sampling on $n$-gram features.
3) DQ~\cite{dq}: dataset quantization with submodular bin selection.
4) Full Finetuning (Full FT): LoRA finetuning on the entire training dataset.
5) SHED-QWCS (SHED) and SHED-QOCS (SHED-O)~\cite{shed}: two
clustering-based Shapley variants.
6) LESS~\cite{less}: optimizer-aware influence estimation using low-rank gradient similarity.
7) NICE~\cite{nice}: policy-gradient influence estimation targeting the downstream evaluation metric.
All budgeted methods share a $10{,}000$-example maximum: the baselines use the full budget, whereas HARP treats it as a cap and dynamically selects its best prefix.

\noindent\textbf{Hyperparameters and reproducibility.}
For final finetuning across methods, we use LoRA (rank $16$, $\alpha=32$) with AdamW, a learning rate of $2\times10^{-4}$, an effective batch size of $16$, and $3$ epochs, averaging results over three seeds.
Details are provided in Appendix Tables~\ref{tab:hp_shared} and~\ref{tab:hp_harp}.

\section{Experimental Results}
In this paper, we first report main results in \S~\ref{sec:exp_main}, and then discuss the data efficiency in \S ~\ref{sec:exp_efficiency}, and finally show the ablation studies in \S~\ref{sec:exp_ablation}.

\subsection{Main Results}
\label{sec:exp_main}

\begin{table}[!t]
\centering
\caption{Main results (accuracy\%). \textbf{Bold}/\underline{underline}: best/second-best among data selectors, excluding the Base, Random, and Full FT reference columns.}
\label{tab:main_results}
\setlength{\tabcolsep}{2pt}
\renewcommand{\arraystretch}{0.90}
\resizebox{0.97\textwidth}{!}{%
\begin{tabular}{@{}llccccccccc@{\hspace{4pt}}cc@{}}
\toprule
& & & \multicolumn{8}{c}{\textbf{Baselines}} & \multicolumn{2}{c}{\textbf{HARP}} \\
\cmidrule(lr){4-11} \cmidrule(lr){12-13}
\textbf{Train} & \textbf{Eval} &\textbf{Base} & \textbf{Random} & \textbf{DSIR} & \textbf{SHED} & \textbf{LESS} & \textbf{NICE} & \textbf{DQ} & \textbf{SHED-O} & \textbf{Full FT} & \textbf{HARP-C} & \textbf{HARP-E} \\
\midrule
\multicolumn{13}{l}{\textbf{Llama-8B}} \\
\multirow{4}{*}{SI} & MMLU & $60.1_{\pm0.0}$ & $58.2_{\pm0.9}$ & $56.1_{\pm0.9}$ & $55.3_{\pm3.1}$ & $47.6_{\pm1.0}$ & \underline{$62.5_{\pm0.4}$} & $8.3_{\pm12.6}$ & $45.5_{\pm23.6}$ & $49.1_{\pm2.9}$ & $\boldsymbol{63.2_{\pm0.3}}$ & $\boldsymbol{63.2_{\pm0.2}}$ \\
 & ARC & $73.3_{\pm0.2}$ & $70.8_{\pm2.1}$ & $71.0_{\pm0.6}$ & $69.0_{\pm3.7}$ & $66.8_{\pm2.7}$ & $74.2_{\pm1.2}$ & $9.0_{\pm10.5}$ & $26.1_{\pm40.0}$ & $58.4_{\pm4.1}$ & \underline{$75.6_{\pm2.4}$} & $\boldsymbol{76.3_{\pm0.9}}$ \\
 & GSM8K & $17.0_{\pm0.7}$ & $8.2_{\pm1.5}$ & $10.5_{\pm0.8}$ & $10.4_{\pm0.6}$ & $6.4_{\pm0.8}$ & $\boldsymbol{26.6_{\pm1.0}}$ & $9.3_{\pm2.2}$ & $9.4_{\pm1.1}$ & $6.7_{\pm0.5}$ & \underline{$23.2_{\pm2.0}$} & $20.6_{\pm2.4}$ \\
 & MATH & $1.6_{\pm0.3}$ & $4.2_{\pm0.5}$ & $4.7_{\pm0.4}$ & $2.8_{\pm0.5}$ & $7.6_{\pm2.0}$ & $6.3_{\pm0.9}$ & $3.8_{\pm1.2}$ & $2.7_{\pm0.6}$ & $2.9_{\pm0.4}$ & \underline{$7.7_{\pm1.2}$} & $\boldsymbol{8.0_{\pm1.1}}$ \\
\cmidrule(lr){1-2} \cmidrule(lr){3-13}
\multirow{4}{*}{Wiz} & MMLU & $60.1_{\pm0.0}$ & $62.8_{\pm0.4}$ & $61.6_{\pm0.4}$ & $61.7_{\pm0.4}$ & $62.7_{\pm0.3}$ & $62.2_{\pm0.7}$ & $62.5_{\pm0.1}$ & $62.9_{\pm0.8}$ & $59.5_{\pm0.4}$ & $\boldsymbol{64.5_{\pm0.1}}$ & \underline{$63.0_{\pm0.2}$} \\
 & ARC & $73.3_{\pm0.2}$ & $77.0_{\pm1.2}$ & $76.4_{\pm1.0}$ & $76.9_{\pm0.7}$ & $76.3_{\pm0.4}$ & $77.2_{\pm0.5}$ & $74.4_{\pm1.6}$ & $74.4_{\pm1.1}$ & $72.2_{\pm1.4}$ & $\boldsymbol{77.9_{\pm0.5}}$ & \underline{$77.6_{\pm0.4}$} \\
 & GSM8K & $17.0_{\pm0.7}$ & $49.2_{\pm2.7}$ & $46.4_{\pm2.3}$ & \underline{$50.2_{\pm0.7}$} & $42.6_{\pm4.1}$ & $49.1_{\pm4.0}$ & $48.5_{\pm2.5}$ & $\boldsymbol{51.4_{\pm2.6}}$ & $51.9_{\pm1.8}$ & $49.7_{\pm2.4}$ & $48.8_{\pm1.5}$ \\
 & MATH & $1.6_{\pm0.3}$ & $10.3_{\pm4.1}$ & $10.8_{\pm0.7}$ & $10.6_{\pm1.0}$ & $10.0_{\pm1.3}$ & $10.4_{\pm1.1}$ & $10.7_{\pm3.0}$ & \underline{$11.0_{\pm0.7}$} & $9.5_{\pm1.2}$ & $8.5_{\pm2.1}$ & $\boldsymbol{12.3_{\pm1.3}}$ \\
\cmidrule(lr){1-2} \cmidrule(lr){3-13}
\multirow{4}{*}{Tulu} & MMLU & $60.1_{\pm0.0}$ & $62.4_{\pm0.2}$ & $62.7_{\pm0.5}$ & $62.7_{\pm0.0}$ & $61.8_{\pm1.0}$ & $63.1_{\pm0.6}$ & $62.4_{\pm0.7}$ & $62.5_{\pm1.2}$ & $60.2_{\pm0.5}$ & $\boldsymbol{64.5_{\pm0.2}}$ & \underline{$63.8_{\pm0.6}$} \\
 & ARC & $73.3_{\pm0.2}$ & $74.7_{\pm1.5}$ & $74.9_{\pm0.2}$ & \underline{$76.3_{\pm1.3}$} & $75.2_{\pm0.5}$ & $76.1_{\pm0.7}$ & $72.7_{\pm1.9}$ & $73.9_{\pm0.9}$ & $72.2_{\pm1.2}$ & $\boldsymbol{79.0_{\pm0.4}}$ & $76.1_{\pm0.7}$ \\
 & GSM8K & $17.0_{\pm0.7}$ & $29.7_{\pm10.9}$ & $30.5_{\pm21.2}$ & $20.0_{\pm6.3}$ & $21.8_{\pm16.1}$ & $\boldsymbol{62.5_{\pm3.0}}$ & $29.3_{\pm3.5}$ & $38.8_{\pm24.6}$ & $57.2_{\pm5.5}$ & $57.9_{\pm0.6}$ & \underline{$58.4_{\pm2.2}$} \\
 & MATH & $1.6_{\pm0.3}$ & $16.4_{\pm2.1}$ & $13.5_{\pm1.1}$ & $15.1_{\pm1.9}$ & $11.2_{\pm0.6}$ & $\boldsymbol{19.3_{\pm1.8}}$ & $15.1_{\pm0.1}$ & $15.8_{\pm1.4}$ & $14.3_{\pm0.9}$ & \underline{$17.4_{\pm1.6}$} & \underline{$17.4_{\pm1.3}$} \\
\midrule
\multicolumn{13}{l}{\textbf{Qwen3-8B}} \\
\multirow{4}{*}{SI} & MMLU & $62.4_{\pm0.1}$ & $70.3_{\pm2.0}$ & $68.2_{\pm1.3}$ & $68.7_{\pm2.0}$ & $66.2_{\pm1.4}$ & $71.4_{\pm0.7}$ & $29.1_{\pm4.6}$ & $64.0_{\pm9.6}$ & $64.8_{\pm1.3}$ & $\boldsymbol{75.4_{\pm0.2}}$ & \underline{$73.6_{\pm0.7}$} \\
 & ARC & $72.9_{\pm0.5}$ & $88.3_{\pm1.0}$ & $87.0_{\pm0.9}$ & $86.5_{\pm0.8}$ & $86.0_{\pm1.8}$ & $\boldsymbol{92.8_{\pm0.1}}$ & $39.5_{\pm16.7}$ & $69.3_{\pm18.3}$ & $82.4_{\pm1.4}$ & $90.3_{\pm0.4}$ & \underline{$91.4_{\pm0.8}$} \\
 & GSM8K & $75.2_{\pm0.6}$ & $12.1_{\pm4.6}$ & $8.5_{\pm0.9}$ & $9.6_{\pm1.9}$ & $8.3_{\pm1.2}$ & $16.8_{\pm0.9}$ & $12.3_{\pm2.7}$ & $12.5_{\pm2.5}$ & $8.8_{\pm2.4}$ & $\boldsymbol{76.2_{\pm1.51}}$ & \underline{$40.8_{\pm1.6}$} \\
 & MATH & $34.3_{\pm0.0}$ & $10.8_{\pm1.7}$ & $7.6_{\pm2.4}$ & $10.5_{\pm1.1}$ & $6.1_{\pm2.7}$ & $5.8_{\pm1.3}$ & $11.6_{\pm1.9}$ & $10.6_{\pm2.5}$ & $7.0_{\pm1.5}$ & \underline{$35.6_{\pm3.0}$} & $\boldsymbol{35.8_{\pm2.8}}$ \\
\cmidrule(lr){1-2} \cmidrule(lr){3-13}
\multirow{4}{*}{Wiz} & MMLU & $62.4_{\pm0.1}$ & $75.3_{\pm0.8}$ & $74.6_{\pm0.6}$ & $75.5_{\pm0.5}$ & \underline{$75.6_{\pm0.4}$} & $\boldsymbol{75.7_{\pm0.5}}$ & $68.2_{\pm3.3}$ & $71.5_{\pm2.6}$ & $74.0_{\pm1.2}$ & $75.4_{\pm0.3}$ & $74.9_{\pm0.2}$ \\
 & ARC & $72.9_{\pm0.5}$ & $92.5_{\pm0.4}$ & $92.0_{\pm1.0}$ & $92.0_{\pm0.7}$ & $92.1_{\pm0.4}$ & $92.8_{\pm0.2}$ & $74.2_{\pm14.2}$ & $85.5_{\pm3.1}$ & $91.2_{\pm0.9}$ & \underline{$92.9_{\pm0.3}$} & $\boldsymbol{93.1_{\pm0.4}}$ \\
 & GSM8K & $75.2_{\pm0.6}$ & $83.0_{\pm5.9}$ & $85.1_{\pm0.1}$ & $86.3_{\pm1.5}$ & $47.2_{\pm9.2}$ & $86.6_{\pm0.3}$ & $86.5_{\pm0.8}$ & \underline{$87.2_{\pm1.2}$} & $84.4_{\pm1.5}$ & $\boldsymbol{89.3_{\pm1.8}}$ & $85.2_{\pm4.2}$ \\
 & MATH & $34.3_{\pm0.0}$ & $47.9_{\pm1.0}$ & $45.2_{\pm1.7}$ & $47.2_{\pm1.4}$ & $47.5_{\pm0.5}$ & \underline{$50.1_{\pm1.0}$} & $46.3_{\pm1.7}$ & $46.7_{\pm1.6}$ & $46.3_{\pm1.0}$ & $48.2_{\pm0.2}$ & $\boldsymbol{51.6_{\pm1.2}}$ \\
\cmidrule(lr){1-2} \cmidrule(lr){3-13}
\multirow{4}{*}{Tulu} & MMLU & $62.4_{\pm0.1}$ & $75.9_{\pm0.1}$ & $75.1_{\pm0.3}$ & \underline{$75.8_{\pm0.7}$} & $74.9_{\pm0.3}$ & $\boldsymbol{75.9_{\pm0.5}}$ & $72.2_{\pm2.5}$ & $72.6_{\pm2.5}$ & $75.0_{\pm0.4}$ & \underline{$75.8_{\pm0.0}$} & $75.7_{\pm0.3}$ \\
 & ARC & $72.9_{\pm0.5}$ & $92.4_{\pm0.5}$ & $\boldsymbol{92.9_{\pm0.3}}$ & $92.6_{\pm0.4}$ & $92.6_{\pm0.4}$ & \underline{$92.7_{\pm0.3}$} & $85.5_{\pm4.4}$ & $87.5_{\pm3.8}$ & $91.7_{\pm0.4}$ & $92.2_{\pm0.2}$ & $92.4_{\pm0.2}$ \\
 & GSM8K & $75.2_{\pm0.6}$ & $59.8_{\pm23.8}$ & \underline{$74.4_{\pm5.3}$} & $63.4_{\pm8.4}$ & $47.8_{\pm13.8}$ & $70.5_{\pm9.4}$ & $50.2_{\pm16.6}$ & $46.6_{\pm18.9}$ & $61.0_{\pm15.1}$ & $74.3_{\pm2.7}$ & $\boldsymbol{76.1_{\pm4.7}}$ \\
 & MATH & $34.3_{\pm0.0}$ & $52.9_{\pm2.0}$ & $47.1_{\pm1.7}$ & $52.0_{\pm1.8}$ & $40.9_{\pm2.4}$ & \underline{$54.2_{\pm1.2}$} & $51.8_{\pm0.7}$ & $49.7_{\pm0.7}$ & $48.2_{\pm1.7}$ & $43.7_{\pm3.5}$ & $\boldsymbol{54.4_{\pm0.7}}$ \\
\midrule
\multicolumn{13}{l}{\textbf{Qwen3-4B}} \\
\multirow{4}{*}{SI} & MMLU & $52.1_{\pm0.2}$ & $68.1_{\pm0.5}$ & $67.0_{\pm1.1}$ & $65.9_{\pm0.9}$ & $62.1_{\pm1.7}$ & $61.6_{\pm1.7}$ & $30.2_{\pm17.7}$ & $64.1_{\pm4.3}$ & $63.1_{\pm1.0}$ & $\boldsymbol{71.4_{\pm0.2}}$ & \underline{$70.1_{\pm0.2}$} \\
 & ARC & $46.9_{\pm1.0}$ & $85.6_{\pm0.5}$ & $84.1_{\pm0.7}$ & $84.1_{\pm0.9}$ & $81.6_{\pm0.8}$ & $\boldsymbol{88.8_{\pm0.3}}$ & $29.9_{\pm25.1}$ & $81.9_{\pm4.4}$ & $80.8_{\pm1.6}$ & $87.5_{\pm0.2}$ & \underline{$87.6_{\pm0.3}$} \\
 & GSM8K & $70.8_{\pm0.1}$ & $13.0_{\pm4.1}$ & $11.1_{\pm2.0}$ & $21.2_{\pm5.2}$ & $10.1_{\pm1.2}$ & $14.7_{\pm3.2}$ & $18.2_{\pm5.4}$ & $32.6_{\pm18.4}$ & $6.8_{\pm0.8}$ & $\boldsymbol{75.1_{\pm4.9}}$ & \underline{$75.0_{\pm5.0}$} \\
 & MATH & $33.1_{\pm0.8}$ & $13.0_{\pm0.7}$ & $7.2_{\pm0.7}$ & $10.3_{\pm0.4}$ & $4.2_{\pm2.4}$ & $10.8_{\pm0.7}$ & $12.7_{\pm2.2}$ & $13.6_{\pm1.3}$ & $9.4_{\pm1.5}$ & $\boldsymbol{34.8_{\pm3.6}}$ & \underline{$34.6_{\pm3.3}$} \\
\cmidrule(lr){1-2} \cmidrule(lr){3-13}
\multirow{4}{*}{Wiz} & MMLU & $52.1_{\pm0.2}$ & $72.2_{\pm0.2}$ & $71.7_{\pm0.1}$ & $72.2_{\pm0.0}$ & $72.0_{\pm0.3}$ & $72.2_{\pm0.3}$ & $53.0_{\pm11.3}$ & $54.6_{\pm3.9}$ & $71.1_{\pm0.6}$ & $\boldsymbol{72.7_{\pm0.2}}$ & \underline{$72.5_{\pm0.2}$} \\
 & ARC & $46.9_{\pm1.0}$ & $88.3_{\pm0.3}$ & \underline{$89.0_{\pm0.3}$} & $88.6_{\pm0.6}$ & $88.7_{\pm0.3}$ & \underline{$89.0_{\pm0.2}$} & $61.2_{\pm19.3}$ & $54.6_{\pm10.2}$ & $88.1_{\pm0.1}$ & $88.3_{\pm0.1}$ & $\boldsymbol{89.2_{\pm0.2}}$ \\
 & GSM8K & $70.8_{\pm0.1}$ & $84.3_{\pm2.7}$ & $82.8_{\pm3.4}$ & $85.7_{\pm1.5}$ & $84.2_{\pm0.7}$ & $87.4_{\pm0.2}$ & $85.3_{\pm0.6}$ & $87.7_{\pm1.0}$ & $80.2_{\pm4.1}$ & \underline{$88.3_{\pm0.4}$} & $\boldsymbol{89.6_{\pm1.1}}$ \\
 & MATH & $33.1_{\pm0.8}$ & $43.8_{\pm0.8}$ & $42.9_{\pm0.2}$ & $44.3_{\pm1.4}$ & $43.3_{\pm1.8}$ & $44.4_{\pm1.8}$ & $43.3_{\pm1.2}$ & $43.7_{\pm1.4}$ & $42.8_{\pm1.5}$ & $\boldsymbol{49.1_{\pm0.8}}$ & \underline{$47.4_{\pm1.6}$} \\
\cmidrule(lr){1-2} \cmidrule(lr){3-13}
\multirow{4}{*}{Tulu} & MMLU & $52.1_{\pm0.2}$ & $72.2_{\pm0.2}$ & $71.0_{\pm0.8}$ & \underline{$72.2_{\pm0.2}$} & $72.0_{\pm0.3}$ & $\boldsymbol{72.5_{\pm0.3}}$ & $71.5_{\pm0.7}$ & $64.6_{\pm5.3}$ & $70.6_{\pm0.5}$ & \underline{$72.2_{\pm0.2}$} & \underline{$72.2_{\pm0.1}$} \\
 & ARC & $46.9_{\pm1.0}$ & $88.3_{\pm0.4}$ & $87.7_{\pm0.1}$ & $88.3_{\pm0.3}$ & $88.1_{\pm0.2}$ & $\boldsymbol{89.1_{\pm0.3}}$ & $86.4_{\pm0.4}$ & $85.5_{\pm1.6}$ & $87.9_{\pm0.8}$ & $88.0_{\pm0.2}$ & \underline{$88.6_{\pm0.2}$} \\
 & GSM8K & $70.8_{\pm0.1}$ & $81.6_{\pm4.5}$ & $79.3_{\pm2.4}$ & $79.0_{\pm2.4}$ & $34.0_{\pm7.7}$ & $81.6_{\pm2.7}$ & $75.7_{\pm4.0}$ & $\boldsymbol{83.0_{\pm2.1}}$ & $74.9_{\pm3.7}$ & $79.6_{\pm2.1}$ & \underline{$82.4_{\pm6.0}$} \\
 & MATH & $33.1_{\pm0.8}$ & $45.8_{\pm1.2}$ & $44.3_{\pm0.2}$ & \underline{$48.7_{\pm1.7}$} & $43.4_{\pm1.2}$ & $\boldsymbol{49.6_{\pm4.0}}$ & $44.5_{\pm0.8}$ & $44.8_{\pm2.6}$ & $44.7_{\pm2.8}$ & $44.3_{\pm1.4}$ & $47.9_{\pm0.6}$ \\
\midrule
\multicolumn{2}{l}{\textit{Mean (36 settings)}} & 50.0 & 56.7 & 56.0 & 56.4 & 51.6 & 59.3 & 45.7 & 52.5 & 54.7 & $\boldsymbol{64.2}$ & \underline{63.5} \\
\bottomrule
\end{tabular}}
\end{table}

Table~\ref{tab:main_results} reports three-seed accuracy across
$3\text{ models}\times 3\text{ training sets}\times 4\text{ evaluation sets}=36$
settings. Among the ranked methods, the HARP variants achieve the highest accuracy in
$\mathbf{23}$ of the $36$ settings. \textsc{HARP-C} is best in $13$ settings and
\textsc{HARP-E} in $11$ compared with $10$ for the strongest baseline
\textsc{NICE}, $2$ for SHED-O, and $1$ for DSIR. Averaged across all settings
\textsc{HARP-C} and \textsc{HARP-E} reach $64.2$ and $63.5$
respectively. Compared with the untuned Base model ($50.0$),
\textsc{HARP-C} and \textsc{HARP-E} improve mean accuracy by $14.2$ and $13.5$ percentage points, respectively. \textsc{HARP-C} and \textsc{HARP-E} exceed the strongest baseline \textsc{NICE} by $+4.9$ and $+4.2$ points and
\textsc{SHED} by $+7.8$ and $+7.1$ points.
For example, on \textsc{Qwen3-4B}/SI/GSM8K, \textsc{HARP-C}
achieves $75.1$, compared with $32.6$ for the strongest ranked baseline, \textsc{SHED-O}.
On \textsc{Qwen3-8B}/SI/GSM8K, it reaches $76.2$ compared with $16.8$ for \textsc{NICE}.
HARP also performs well beyond SI: on Wiz, \textsc{HARP-C} and
\textsc{HARP-E} both achieve a mean accuracy of $67.1$,
compared with $66.4$ for \textsc{NICE}; on Tulu,
\textsc{HARP-E} nearly matches \textsc{NICE} ($67.1$ versus $67.3$).
This is because HARP's hierarchy divides the candidate pool into small local leaf groups, and interpolates the measured main effects of representative leaves to obtain local estimates for unmeasured leaves within each parent node. This guides the ranking to favor positively contributing leaves over weakly contributing or potentially harmful ones. Such selectivity may be particularly beneficial on SI where the candidate pool may contain many irrelevant or misleading examples.

\textsc{HARP-C} achieves higher mean accuracy on SI ($59.7$ versus $56.4$) whereas \textsc{HARP-E} performs better on Tulu ($67.1$ versus $65.7$). The two variants are nearly tied on Wiz (both $67.1$ after rounding).  This pattern is consistent with how their envelopes account for positive contributions. Both accumulate estimated harms but \textsc{HARP-C} credits only the strongest positive leaf within each domain. Its conservative coverage envelope therefore limits additional credit for leaves with overlapping benefits. \textsc{HARP-E} instead sums signed main effects across leaves through its expansive additive envelope. This preserves credit for complementary contributions within the same domain that \textsc{HARP-C} may omit but can also overcount gains when those contributions are redundant.   These results suggest using \textsc{HARP-C} when controlling redundancy and keeping the selected subset compact are priorities and \textsc{HARP-E} when retaining multiple complementary contributions within the same domain is more important.

\subsection{Data Efficiency}
\label{sec:exp_efficiency}

\begin{wrapfigure}{r}{0.54\textwidth}
\vspace{-\intextsep}
\centering
\includegraphics[width=\linewidth]{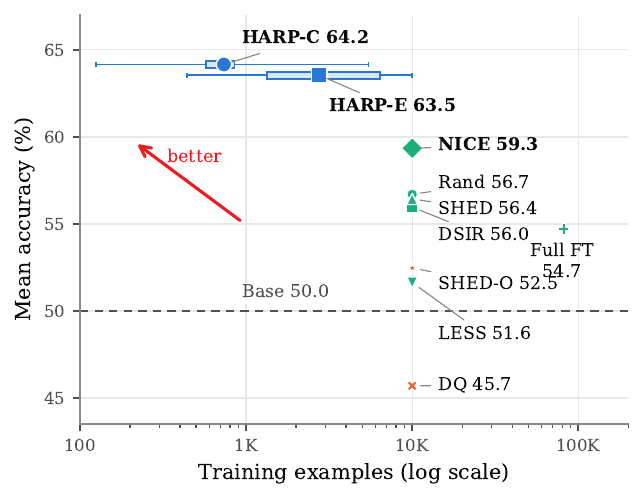}
\vspace{-\intextsep}
\caption{Mean accuracy versus median number of training examples.}
\label{fig:efficiency}
\vspace{-\intextsep}
\end{wrapfigure}

Figure~\ref{fig:efficiency} compares mean accuracy and training set size, using median selected sizes for the HARP markers. All budgeted selection methods share the same maximum budget of $10{,}000$ examples. The baselines use this fixed budget, whereas HARP treats it as a cap and dynamically selects a smaller prefix when its estimated utility peaks earlier. \textsc{HARP-C} achieves the highest mean accuracy with the fewest training examples, using a median of $1{,}087$ examples compared with $4{,}106$ for \textsc{HARP-E}, $10{,}000$ for the budgeted baselines, and $81{,}674$ for full finetuning. Across settings, \textsc{HARP-C} selects roughly $125$--$5{,}467$ examples, while \textsc{HARP-E} ranges from about $442$ to the full $10{,}000$-example budget. Despite using the full pool, Full FT reaches only $54.7$ mean accuracy, trailing HARP-C and HARP-E by $9.5$ and $8.8$ points, respectively. This highlights the importance of selecting useful training data. HARP obtains its smaller subsets through the best prefix rule, which dynamically selects under the same $10{,}000$-example maximum budget rather than using a smaller preset budget. Leaves are ranked greedily by marginal gain under the corresponding envelope, and the prefix with the highest estimated utility is selected. Under \textsc{HARP-C}, once a domain is already covered by a strong positive leaf, additional leaves in the same domain receive no further positive credit unless they provide a stronger contribution, while estimated harms continue to accumulate. Under \textsc{HARP-E}, contributions from multiple leaves are added, so additional leaves can still improve estimated utility when they provide complementary positive effects. Thus, either envelope may reach its highest estimated utility before the $10{,}000$ example cap and select a smaller subset. This occurs more often for \textsc{HARP-C} because redundant positive contributions receive no additional credit. This provides a plausible explanation for the observed combination of compact subsets and strong downstream performance, further highlighting the importance of efficient data selection.

\subsection{Ablation Studies}
\label{sec:exp_ablation}

In this section, we ablate the number of representatives per node $R$, leaf size $L$, and hierarchy partitioning on \textsc{Qwen3-4B-Base} with \textsc{WizardLM}, reporting mean accuracy and standard deviation over three seeds.

\begin{figure}[t]
\centering
\includegraphics[width=\linewidth]{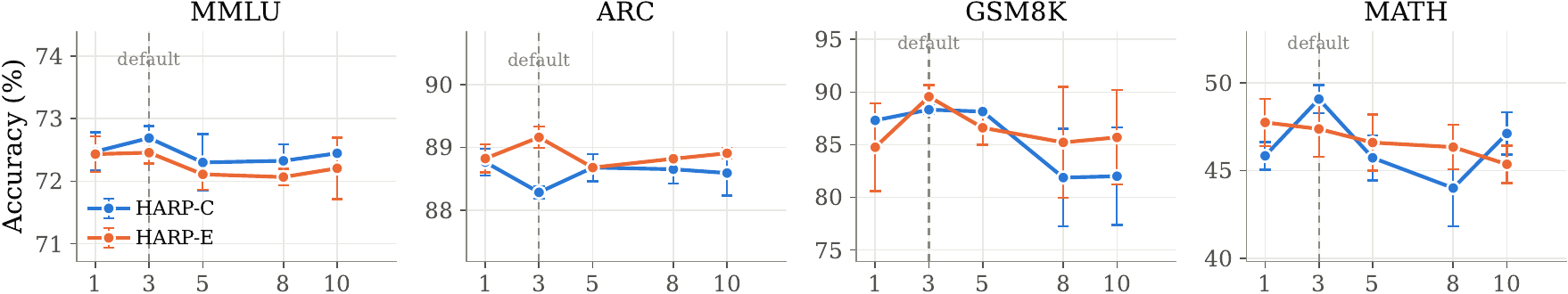}
\caption{Effect of representatives per node $R$.}
\label{fig:ablation_rep}
\end{figure}

\begin{figure}[t]
\centering
\includegraphics[width=\linewidth]{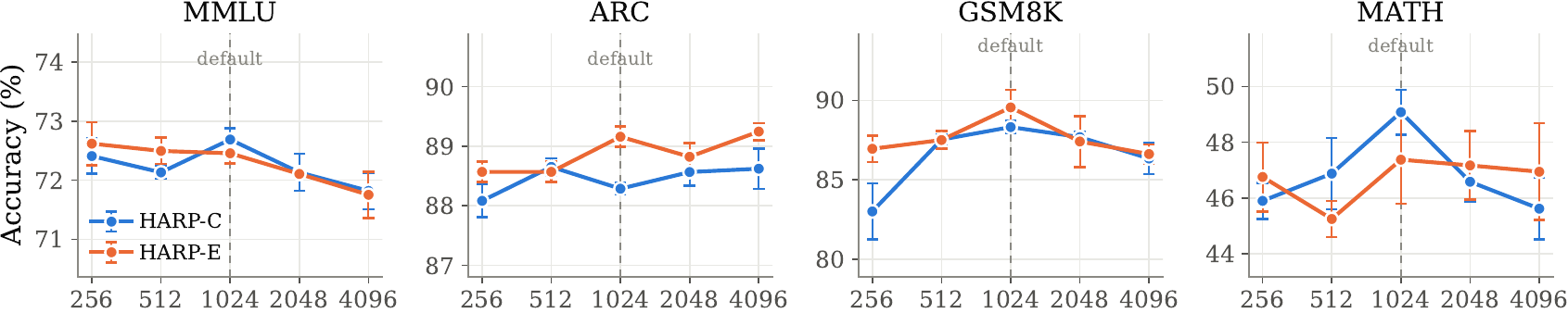}
\caption{Effect of leaf size $L$.}
\label{fig:ablation_leaf}
\end{figure}

\begin{figure}[!t]
\centering
\includegraphics[width=\linewidth]{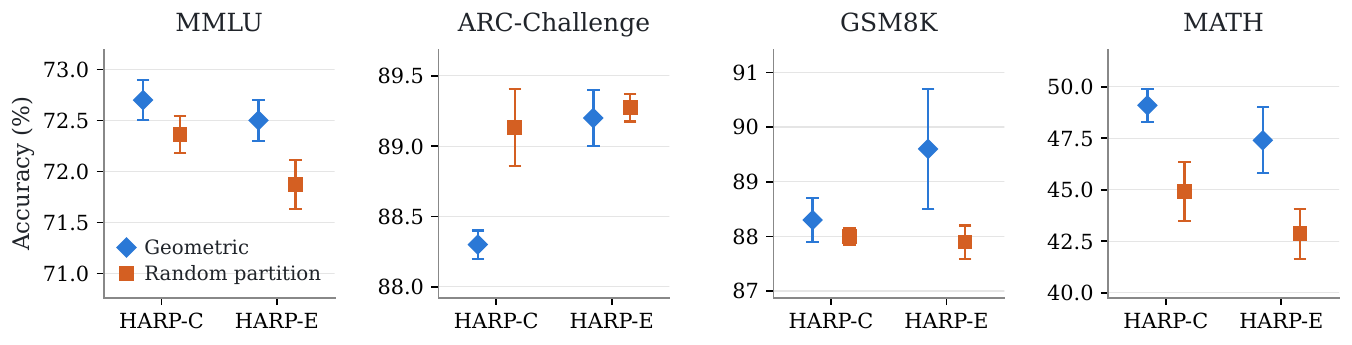}
\caption{Random partition versus geometric grouping.}
\label{fig:ablation_random_partition}
\end{figure}

\noindent\textbf{Effectiveness of representatives per node $R$.}
We vary $R\in\{1,3,5,8,10\}$ (Figure~\ref{fig:ablation_rep}). Larger $R$ directly measures more representative leaves but increases train--evaluate cost (Appendix~\ref{app:complexity}). We find that MMLU and ARC remain within $0.5$ points of the $R=3$ setting for both variants, while larger $R$ does not consistently improve GSM8K or MATH. The strong performance at small $R$ supports the effectiveness of our estimation procedure. Interpolation transfers local information from measured representatives within the same parent node, while empirical-Bayes shrinkage stabilizes the resulting estimates when representative support is limited. This allows HARP to preserve downstream performance while directly measuring only a small number of leaves. We therefore use $R=3$ as the default to balance accuracy and measurement cost.

\noindent\textbf{Effectiveness of leaf size $L$.}
We vary $L\in\{256,512,1024,2048,4096\}$ (Figure~\ref{fig:ablation_leaf}). Both variants remain within $0.9$ points of the $L=1024$ setting on MMLU and ARC, while GSM8K and MATH are more sensitive. Increasing $L$ from $256$ to $1024$ improves \textsc{HARP-C} from $83.0$ to $88.3$ on GSM8K and from $45.9$ to $49.1$ on MATH, with little additional gain at larger values. \textsc{HARP-E} also reaches its highest GSM8K and MATH accuracy at $L=1024$. The broad stability across a $16\times$ range of $L$ supports the robustness of HARP's hierarchy to leaf granularity. The coverage-oriented partitioning preserves diverse local regions, while the leaf-size constraints keep selection units sufficiently large for stable utility estimation. Very small leaves can therefore become less stable, while larger leaves make the selection coarser without providing consistent gains. We use $L=1024$ as the default.

\noindent\textbf{Effectiveness of geometric grouping.}
We compare geometric grouping with random partitioning on \textsc{Qwen3-4B-Base} with \textsc{WizardLM}. Geometric grouping yields higher accuracy on MMLU, GSM8K, and MATH for both variants, while random partitioning is slightly better on ARC. On MATH, geometric grouping improves \textsc{HARP-C} from $44.9$ to $49.1$ and \textsc{HARP-E} from $42.9$ to $47.4$. Averaged across the four benchmarks and both variants, the gain is $1.3$ points. Geometric grouping keeps related leaves local, improving representative interpolation and producing coherent selection units with clearer gain and harm estimates. Random partitioning mixes heterogeneous regions, weakening both utility estimation and greedy ranking.

\section{Conclusion}
\label{sec:conclusion}

We propose HARP, a hierarchical data selection method with two complementary
envelopes, \textsc{HARP-C} for conservative pruning and \textsc{HARP-E} for
broader coverage. Across $36$ settings, HARP achieves the best result among
ranked methods in $23$ settings. \textsc{HARP-C} and \textsc{HARP-E} reach mean
accuracies of $64.2$ and $63.5$, respectively, compared with $59.3$ for
\textsc{NICE} and $54.7$ for full finetuning. \textsc{HARP-C} uses roughly
$9\times$ fewer training examples than the budgeted baselines and $75\times$
fewer than full finetuning. The two envelopes show complementary behavior.
\textsc{HARP-C} is strongest on SI, with mean accuracy $59.7$, while
\textsc{HARP-E} performs better on Tulu and the two
variants are nearly tied on Wiz. Overall, these results show that HARP can
identify compact training subsets while preserving strong downstream
performance across different data regimes.

\subsection*{AI use statement}

We used AI assistants during the preparation of this work for limited coding and writing support. In particular, we use Claude Code for parts of the implementation, including code drafting, debugging, and refactoring. We also used GPT/OpenAI tools to improve the clarity, grammar, and presentation of the manuscript. All research ideas, method design, experimental decisions, result analysis, and final manuscript content were reviewed and verified by the authors. The authors take full responsibility for the correctness, originality, and integrity of the submitted work.

\bibliographystyle{iclr2027_conference}
\bibliography{references}

\clearpage
\appendix
\section{Artifact Use Consistent With Intended Use}
\label{sec:artifact_use}

We use pretrained models, finetuning datasets, evaluation benchmarks, and baseline implementations described in Section~\ref{sec:experiments} only for research purposes. Our use is limited to controlled data-selection, finetuning, and benchmark evaluation experiments, consistent with the intended research use of these artifacts. Artifacts produced by this work, including selected training subsets, hierarchy partitions, utility estimates, and code, are intended for research on efficient finetuning data selection, and any derived artifacts should be used or redistributed only when permitted by the licenses and access terms of the original datasets, models, and benchmarks.

\section{Proof}
\subsection{Universal Notation Table}
\label{app:notation_table}

Table~\ref{tab:universal_notation} summarizes the notation used in the method and appendix. The same symbols are used throughout unless a local proof explicitly introduces a temporary dummy variable.

\begin{longtable}{p{0.26\linewidth}p{0.68\linewidth}}
\caption{Universal notation table for HARP.}
\label{tab:universal_notation}\\
\hline
\textbf{Symbol} & \textbf{Meaning}\\
\hline
\endfirsthead
\hline
\textbf{Symbol} & \textbf{Meaning}\\
\hline
\endhead
\hline
\endfoot
\multicolumn{2}{l}{\textit{Training data and full-set validation}}\\
\(\mathcal{D}=\{x_i\}_{i=1}^N\) & Candidate training dataset; \(x_i\) is one training example and \(N\) is the number of candidates.\\
\(\mathcal{E}\), \(\mathcal{E}_d\), \(D\) & Full validation set and its domain-\(d\) subset, and the number of validation domains. The index \(d\in[D]\) ranges over validation domains.\\
\(n_d\) & Number of examples in the full validation domain \(\mathcal{E}_d\).\\
\(\hat z_i\), \(\bar z_g\) & Normalized training-example embedding and mean embedding of leaf \(g\).\\
\(A(S)\), \(c(S)\), \(c_g\) & Training subset induced by selected leaves, total selected leaf cost, and individual leaf cost: \(A(S)=\bigcup_{g\in S}\mathcal{L}_g\), \(c(S)=\sum_{g\in S}|\mathcal{L}_g|\), and \(c_g=|\mathcal{L}_g|\).\\
\(B\), \(K\) & Training-data budget and an upper bound on the number of leaves in any feasible set.\\
\(C_{\min}\), \(C_{\max}\) & Minimum leaf-size target and maximum leaf-size cap.\\
\(U(S)\), \(u_d(S)\), \(v_0(d)\), \(b_d\) & Task-averaged utility, domain utility, base-model domain utility, and shorthand \(b_d=v_0(d)\).\\
\(\mathcal{L}_g\), \(G\), \(p(g)\), \(\mathcal{P}_p\) & Leaf group, number of leaves, parent node of leaf \(g\), and leaves under parent \(p\).\\
\(\mathcal{R}_p\), \(r_p\), \(R\), \(Q\) & Representative leaves in parent \(p\), their count, per-parent representative cap, and total measured representative count \(Q=\sum_p r_p\).\\
\(\phi_g(d)\), \(y_r(d)\), \(\widetilde y_g(d)\), \(\widehat\phi_g(d)\) & True main effect, measured representative effect, interpolated estimate, and final estimated main effect.\\
\(\alpha_{g,r}\), \(n_{\mathrm{eff}}(g)\), \(\Delta_{g,p}\) & Interpolation weight, effective representative count, and weighted representative distance.\\
\(\mu_0(d)\), \(\tau^2(d)\), \(\sigma_p^2(d)\), \(\rho_g(d)\) & Global measured-effect mean, across-node variance, within-node variance, and empirical-Bayes shrinkage weight.\\
\(\mathcal{A}_{\mathrm{dom}}\), \(w_d\), \(\epsilon_{\mathrm{dom}}\) & Active domains, active-domain weights, and active-domain threshold.\\
\(\mathcal{F}_k\) & Feasible candidate set used by greedy selection at step \(k\): \(\mathcal{F}_k=\{g\notin S_k:c(S_k)+c_g\le B\}\).\\
\(y_d^{\mathrm{C}}(S)\), \(\widehat y_d^{\mathrm{C}}(S)\) & True and estimated clipped conservative domain utilities.\\
\(U_{\mathrm{C}}(S)\), \(\widehat U_{\mathrm{C}}(S)\) & True and estimated conservative first-order utilities.\\
\(y_d^{\mathrm{E}}(S)\), \(\widehat y_d^{\mathrm{E}}(S)\) & True and estimated clipped expansive additive domain utilities.\\
\(U_{\mathrm{E}}(S)\), \(\widehat U_{\mathrm{E}}(S)\) & True and estimated expansive additive first-order utilities.\\
\(q_d^{\mathrm{C}}(S)\), \(q_d^{\mathrm{E}}(S)\) & Raw unclipped conservative and expansive domain envelopes.\\
\(\epsilon_{\mathrm{C}}\), \(\delta_{\mathrm{C}}\), \(\epsilon_{\mathrm{E}}\), \(\delta_{\mathrm{E}}\), \(\epsilon_{\mathrm{opt}}\) & First-order surrogate errors, estimation errors, and optimization gap for HARP-C and HARP-E.\\
\end{longtable}

\subsection{Iteration-Reduction Bound for Leaf-Level Estimation}
\label{app:hierarchy_cost}

The following lemma isolates the reduction obtained by moving from individual examples to leaf groups. It bounds the number of the utility queries that would be required if every final leaf were measured once; the representative-leaf estimation step in Section~\ref{subsec:main_effect_estimation} further reduces this cost in practice. The upper leaf-size parameter \(C_{\max}\) controls recursive splitting, while the lower bound \(C_{\min}\) controls the counting argument below.

\begin{lemma}[Leaf grouping reduces the utility-query count]
\label{lem:hierarchy_cost}
Let \(\{\mathcal{L}_g\}_{g=1}^{G}\) be the final leaf groups constructed by HARP. Suppose that the leaves form a partition of the candidate pool \(\mathcal{D}\), where \(|\mathcal{D}|=N\), and that each final leaf satisfies the minimum-size constraint \(|\mathcal{L}_g|\ge C_{\min}\). Then
\begin{equation*}
G\le \frac{N}{C_{\min}}.
\end{equation*}
Therefore, exhaustive utility evaluation at the leaf level requires at most \(N/C_{\min}\) train--evaluate iterations, compared with \(N\) iterations for exhaustive individual-example singleton estimation.
\end{lemma}

\begin{proof}
Because the final leaf groups form a partition of \(\mathcal{D}\), each training example belongs to exactly one leaf group, and therefore
\begin{equation*}
N=\sum_{g=1}^{G}|\mathcal{L}_g|.
\end{equation*}
By the minimum-size constraint, \(|\mathcal{L}_g|\ge C_{\min}\) for every \(g\). Hence
\begin{equation*}
N
=
\sum_{g=1}^{G}|\mathcal{L}_g|
\ge
\sum_{g=1}^{G}C_{\min}
=
G C_{\min}.
\end{equation*}
Rearranging gives \(G\le N/C_{\min}\). Thus, replacing individual examples with leaf groups reduces the number of possible singleton utility measurements from \(N\) to at most \(N/C_{\min}\) before applying any additional representative-leaf subsampling.
\end{proof}

\subsection{Representative Interpolation Error}
\label{app:rep_interpolation}

All utilities in this analysis are defined on the fixed full validation
domains of Section~\ref{subsec:problem_setup}. The measurement-noise model
below conditions on that fixed evaluation set; its zero-mean and
independence conditions are assumptions about stochastic training or
measurement, not consequences of the evaluation protocol.

For the interpolation weights used in Section~\ref{subsec:main_effect_estimation}, assume \(\mathcal{R}_p\neq\varnothing\) for every nonempty parent node \(p\). Define
\begin{equation*}
\nu(\bar z)
=
\begin{cases}
\bar z/\|\bar z\|_2, & \|\bar z\|_2>0,\\
0, & \|\bar z\|_2=0,
\end{cases}
\end{equation*}
and, for \(\lambda>0\),
\begin{equation*}
\kappa(\bar z_g,\bar z_r)
=
\exp\left(\frac{\nu(\bar z_g)^\top\nu(\bar z_r)}{\lambda}\right).
\end{equation*}
Then
\begin{equation*}
\alpha_{g,r}
=
\frac{\kappa(\bar z_g,\bar z_r)}{\sum_{r'\in\mathcal{R}_{p(g)}}\kappa(\bar z_g,\bar z_{r'})}
\end{equation*}
is well-defined, nonnegative, and satisfies \(\sum_{r\in\mathcal{R}_{p(g)}}\alpha_{g,r}=1\).

\begin{lemma}[Representative interpolation under local smoothness]
\label{lem:rep_interpolation}
Fix an nonempty parent node \(p\), a validation domain \(d\), and a leaf \(g\in\mathcal{P}_p\), with \(\mathcal{R}_p\neq\varnothing\). Suppose the true domainwise main effect is locally smooth over leaf embeddings inside parent \(p\): for any leaves \(g,h\in\mathcal{P}_p\),
\begin{equation*}
|\phi_g(d)-\phi_h(d)|
\le
L_{p,d}\|\bar z_g-\bar z_h\|_2 .
\end{equation*}
For each measured representative \(r\in\mathcal{R}_p\), suppose
\begin{equation*}
\begin{aligned}
y_r(d)&=\phi_r(d)+\varepsilon_r(d),
\\
\mathbb{E}[\varepsilon_r(d)]&=0,
\\
\mathrm{Var}(\varepsilon_r(d))&\le \sigma_p^2(d),
\end{aligned}
\end{equation*}
and suppose the noises are independent across representatives. For an unmeasured leaf \(g\), define the local representative interpolation estimate by
\begin{equation*}
\begin{aligned}
\widetilde y_g(d)
&=
\sum_{r\in\mathcal{R}_p}\alpha_{g,r}y_r(d),
\\
\sum_{r\in\mathcal{R}_p}\alpha_{g,r}&=1,
\\
\alpha_{g,r}&\ge 0 .
\end{aligned}
\end{equation*}
Let
\begin{equation*}
\begin{aligned}
\Delta_{g,p}
&=
\sum_{r\in\mathcal{R}_p}
\alpha_{g,r}\|\bar z_g-\bar z_r\|_2,
\\
n_{\mathrm{eff}}(g)
&=
\frac{1}{\sum_{r\in\mathcal{R}_p}\alpha_{g,r}^2}.
\end{aligned}
\end{equation*}
Then the interpolation bias and variance satisfy
\begin{equation*}
\left|
\mathbb{E}[\widetilde y_g(d)]-\phi_g(d)
\right|
\le
L_{p,d}\Delta_{g,p},
\end{equation*}
and
\begin{equation*}
\mathrm{Var}(\widetilde y_g(d))
\le
\frac{\sigma_p^2(d)}{n_{\mathrm{eff}}(g)}.
\end{equation*}
Consequently,
\begin{equation*}
\begin{aligned}
\mathbb{E}\left[(\widetilde y_g(d)-\phi_g(d))^2\right]
\le
L_{p,d}^2\Delta_{g,p}^2
+
\frac{\sigma_p^2(d)}{n_{\mathrm{eff}}(g)}.
\end{aligned}
\end{equation*}
For a directly measured representative leaf, HARP sets \(\widehat\phi_g(d)=y_g(d)\) instead of using interpolation; equivalently, one may take \(\alpha_{g,g}=1\), which makes \(\Delta_{g,p}=0\).
\end{lemma}

\begin{proof}
Using the definition of \(\widetilde y_g(d)\) and the zero-mean noise assumption,
\begin{equation*}
\mathbb{E}[\widetilde y_g(d)]
=
\sum_{r\in\mathcal{R}_p}\alpha_{g,r}\phi_r(d).
\end{equation*}
Since the weights are nonnegative and sum to one,
\begin{equation*}
\begin{aligned}
&\left|
\mathbb{E}[\widetilde y_g(d)]-\phi_g(d)
\right|
\\
&\qquad=
\left|
\sum_{r\in\mathcal{R}_p}
\alpha_{g,r}(\phi_r(d)-\phi_g(d))
\right|
\\
&\qquad\le
\sum_{r\in\mathcal{R}_p}
\alpha_{g,r}|\phi_r(d)-\phi_g(d)|.
\end{aligned}
\end{equation*}
By local smoothness,
\begin{equation*}
|\phi_r(d)-\phi_g(d)|
\le
L_{p,d}\|\bar z_g-\bar z_r\|_2.
\end{equation*}
Combining the two inequalities gives
\begin{equation*}
\begin{aligned}
&\left|
\mathbb{E}[\widetilde y_g(d)]-\phi_g(d)
\right|
\\
&\qquad\le
L_{p,d}
\sum_{r\in\mathcal{R}_p}
\alpha_{g,r}\|\bar z_g-\bar z_r\|_2
\\
&\qquad=
L_{p,d}\Delta_{g,p}.
\end{aligned}
\end{equation*}

For the variance, independence gives
\begin{equation*}
\begin{aligned}
\mathrm{Var}(\widetilde y_g(d))
&=
\mathrm{Var}\left(
\sum_{r\in\mathcal{R}_p}
\alpha_{g,r}\varepsilon_r(d)
\right)
\\
&=
\sum_{r\in\mathcal{R}_p}
\alpha_{g,r}^2\mathrm{Var}(\varepsilon_r(d)).
\end{aligned}
\end{equation*}
Since \(\mathrm{Var}(\varepsilon_r(d))\le \sigma_p^2(d)\),
\begin{equation*}
\begin{aligned}
\mathrm{Var}(\widetilde y_g(d))
&\le
\sigma_p^2(d)
\sum_{r\in\mathcal{R}_p}\alpha_{g,r}^2
\\
&=
\frac{\sigma_p^2(d)}{n_{\mathrm{eff}}(g)}.
\end{aligned}
\end{equation*}
Finally, the mean-squared interpolation error decomposes into squared bias plus variance:
\begin{equation*}
\begin{aligned}
&\mathbb{E}\left[(\widetilde y_g(d)-\phi_g(d))^2\right]
\\
&=
\left(\mathbb{E}[\widetilde y_g(d)]-\phi_g(d)\right)^2
+
\mathrm{Var}(\widetilde y_g(d)).
\end{aligned}
\end{equation*}
Substituting the bias and variance bounds proves the stated MSE bound. The directly measured case follows by setting all interpolation mass on \(g\), which eliminates the smoothness-bias term.
\end{proof}

\subsection{Empirical-Bayes Shrinkage}
\label{app:eb_shrinkage}

\begin{proposition}[Normal-normal shrinkage and risk reduction]
\label{prop:eb_shrinkage}
Fix an unmeasured leaf \(g\) and domain \(d\). Suppose the local representative estimate satisfies
\begin{equation*}
\widetilde y_g(d)\mid \theta_g(d)
\sim
\mathcal{N}(\theta_g(d),V_g(d)),
\end{equation*}
where \(\theta_g(d)=\mathbb{E}[\widetilde y_g(d)]\) denotes the local interpolation target rather than necessarily the true main effect \(\phi_g(d)\). By Lemma~\ref{lem:rep_interpolation}, this target satisfies \(|\theta_g(d)-\phi_g(d)|\le L_{p(g),d}\Delta_{g,p(g)}\).
where \(V_g(d)=\sigma_{p(g)}^2(d)/n_{\mathrm{eff}}(g)\) and \(V_g(d)>0\), and suppose the latent local interpolation target has prior
\begin{equation*}
\theta_g(d)\sim \mathcal{N}(\mu_0(d),\tau^2(d))\,\quad \tau^2(d)>0.
\end{equation*}
Then the posterior mean is
\begin{equation*}
\begin{aligned}
\widehat{\phi}_g(d)
&=
\rho_g(d)\widetilde y_g(d)
+
(1-\rho_g(d))\mu_0(d),
\\
\rho_g(d)
&=
\frac{\tau^2(d)}{\tau^2(d)+V_g(d)}.
\end{aligned}
\end{equation*}
Moreover, under squared error loss, this posterior mean is the Bayes estimator, and its integrated mean-squared error is
\begin{equation*}
\begin{aligned}
&\mathbb{E}\left[(\widehat{\phi}_g(d)-\theta_g(d))^2\right]
\\
&
=
\frac{\tau^2(d)V_g(d)}{\tau^2(d)+V_g(d)}
\le
V_g(d)
\\
&
=
\mathbb{E}\left[(\widetilde y_g(d)-\theta_g(d))^2\right].
\end{aligned}
\end{equation*}
Consequently, when the later theory uses \(\widehat{\phi}_g(d)\) as an estimator of the true main effect \(\phi_g(d)\), the error bound must include both shrinkage risk and interpolation bias:

\begin{equation*}
\begin{aligned}
&\mathbb{E}\left[(\widehat{\phi}_g(d)-\phi_g(d))^2\right]
\\
&\qquad\le
2\frac{\tau^2(d)V_g(d)}{\tau^2(d)+V_g(d)}
+
2L_{p(g),d}^2\Delta_{g,p(g)}^2 .
\end{aligned}
\end{equation*}

For directly measured leaves, HARP sets \(\widehat\phi_g(d)=y_g(d)\); this shrinkage proposition applies to unmeasured leaves estimated by interpolation.
\end{proposition}

\begin{proof}
For readability, suppress \(g,d\) and write \(\widetilde y\), \(\theta\), \(\mu_0\), \(\tau^2\), and \(V\). The likelihood and prior are
\begin{equation*}
\begin{aligned}
\widetilde y\mid \theta
\sim
\mathcal{N}(\theta,V), \theta
\sim
\mathcal{N}(\mu_0,\tau^2).
\end{aligned}
\end{equation*}
By conjugacy of the normal-normal model, the posterior precision is
\[
\frac{1}{\tau^2}+\frac{1}{V},
\]
so the posterior variance is
\begin{equation*}
\begin{aligned}
\left(\frac{1}{\tau^2}+\frac{1}{V}\right)^{-1}
&=
\frac{\tau^2V}{\tau^2+V}.
\end{aligned}
\end{equation*}
The posterior mean is
\begin{equation*}
\begin{aligned}
&\frac{\frac{1}{V}\widetilde y+\frac{1}{\tau^2}\mu_0}{\frac{1}{V}+\frac{1}{\tau^2}}
=
\frac{\tau^2}{\tau^2+V}\widetilde y
+
\frac{V}{\tau^2+V}\mu_0.
\end{aligned}
\end{equation*}
Thus \(\rho=\tau^2/(\tau^2+V)\), giving the estimator in the proposition.

Under squared error loss, the posterior mean minimizes posterior expected squared error, so it is the Bayes estimator. Its integrated Bayes risk equals the expected posterior variance:
\begin{equation*}
\mathbb{E}\left[(\widehat\phi-\theta)^2\right]
=
\frac{\tau^2V}{\tau^2+V}.
\end{equation*}
The unshrunk estimator \(\widetilde y\) has integrated mean-squared error
\begin{equation*}
\begin{aligned}
\mathbb{E}\left[(\widetilde y-\theta)^2\right]
&=
\mathbb{E}[\varepsilon^2]
=
V.
\end{aligned}
\end{equation*}
Since \(\tau^2/(\tau^2+V)\le 1\),
\[
\frac{\tau^2V}{\tau^2+V}\le V.
\]
Therefore the shrinkage estimator has no larger integrated squared error than the unshrunk local estimate. Finally, the true-effect bound follows from \((a+b)^2\le 2a^2+2b^2\). Since \(\widehat\phi-\phi=(\widehat\phi-\theta)+(\theta-\phi)\) and \(|\theta-\phi|\le L_{p(g),d}\Delta_{g,p(g)}\),

\begin{equation*}
\begin{aligned}
\mathbb{E}\left[(\widehat\phi-\phi)^2\right]
&\le
2\mathbb{E}\left[(\widehat\phi-\theta)^2\right]
+
2(\theta-\phi)^2
\\
&\le
2\frac{\tau^2V}{\tau^2+V}
+
2L_{p(g),d}^2\Delta_{g,p(g)}^2 .
\end{aligned}
\end{equation*}

This proves the additional true main-effect error bound stated in the proposition.
\end{proof}

\begin{proposition}[MSE-to-uniform main-effect error event]
\label{prop:mse_uniform_event}
Suppose that for every leaf-domain pair \((g,d)\) there is a finite bound \(M_{g,d}\) such that
\[
\mathbb{E}\left[(\widehat\phi_g(d)-\phi_g(d))^2\right]
\le
M_{g,d}.
\]
For an unmeasured leaf, Proposition~\ref{prop:eb_shrinkage} gives one admissible choice
\[
M_{g,d}
=
2\frac{\tau^2(d)V_g(d)}{\tau^2(d)+V_g(d)}
+
2L_{p(g),d}^2\Delta_{g,p(g)}^2 .
\]
For a directly measured leaf with \(\widehat\phi_g(d)=y_g(d)=\phi_g(d)+\varepsilon_g(d)\), one may take any valid second-moment bound on \(\varepsilon_g(d)\), for example \(M_{g,d}\ge \mathbb{E}[\varepsilon_g(d)^2]\). Let
\[
M_\Sigma
=
\sum_{g=1}^{G}\sum_{d=1}^{D} M_{g,d}.
\]
Then, for any failure probability \(\alpha\in(0,1)\), with
\[
\eta_\alpha
=
\sqrt{\frac{M_\Sigma}{\alpha}},
\]
we have
\[
\Pr\!\left(
\max_{g\in[G],\,d\in[D]}
|\widehat\phi_g(d)-\phi_g(d)|
\le
\eta_\alpha
\right)
\ge
1-\alpha .
\]
Consequently, the deterministic condition \(|\widehat\phi_g(d)-\phi_g(d)|\le\eta\) used in the envelope sensitivity lemma can be read as conditioning on this uniform event, with \(\eta=\eta_\alpha\).
\begin{proof}
Let
\[
Z
=
\sum_{g=1}^{G}\sum_{d=1}^{D}
(\widehat\phi_g(d)-\phi_g(d))^2 .
\]
By the assumed pairwise MSE bounds, \(\mathbb{E}[Z]\le M_\Sigma\). If
\[
\max_{g,d}|\widehat\phi_g(d)-\phi_g(d)|>\eta_\alpha,
\]
then \(Z>\eta_\alpha^2\). Therefore, by Markov's inequality,
\[
\begin{aligned}
&\Pr\!\left(
\max_{g,d}|\widehat\phi_g(d)-\phi_g(d)|>\eta_\alpha
\right)\\
&\quad\le
\Pr(Z>\eta_\alpha^2)
\le
\frac{\mathbb{E}[Z]}{\eta_\alpha^2}\\
&\quad\le
\frac{M_\Sigma}{M_\Sigma/\alpha}
=
\alpha .
\end{aligned}
\]
Taking complements proves the claim.
\end{proof}
\end{proposition}

\subsection{Train--Evaluate Iteration Reduction from Representative Main-Effect Estimation}
\label{app:main_effect_query_reduction}

\begin{lemma}[Representative main-effect estimation reduces train--evaluate iterations]
\label{lem:main_effect_query_reduction}
Let \(\{\mathcal{L}_g\}_{g=1}^G\) be the training leaves and
\(\{\mathcal{P}_p\}_{p=1}^P\) their parent partition. HARP measures
\(r_p = |\mathcal{R}_p|\) representatives in each nonempty parent, where
\[
1 \le r_p \le |\mathcal{P}_p|,
\qquad
\mathcal{R}_p \subseteq \mathcal{P}_p.
\]
Set
\[
Q = \sum_{p=1}^P r_p.
\]
Main-effect estimation requires \(Q\) train--evaluate iterations instead
of \(G\), saving \(G-Q\) iterations and reducing their count by a factor
of \(G/Q\). If \(r_p \le R\) for every parent, then
\[
Q \le PR.
\]
Each iteration scores its model on the full fixed validation set.
\end{lemma}

\begin{proof}
Exhaustive singleton leaf evaluation fine-tunes one model per leaf and
scores that model on every validation domain, requiring \(G\) iterations.
HARP measures \(r_p\) leaves in parent \(p\), for a total of
\[
Q = \sum_{p=1}^P r_p.
\]
Interpolation and empirical-Bayes shrinkage estimate the other \(G-Q\)
leaves without additional fine-tuning runs, giving the claimed saving
and ratio. Finally,
\[
Q = \sum_p r_p \le \sum_p R = PR.
\]
The common base-model evaluation does not require a fine-tuning run and
is excluded from both iteration counts.
\end{proof}

\subsection{Theory for First-Order Envelope Selection}
\label{app:harp_fo_theory}

We prove a shared stability template for both HARP-C and HARP-E. The two
variants differ only in how estimation error propagates through the envelope:
HARP-C uses a maximum over positive effects and therefore has sensitivity
\((K+1)\eta\), while
HARP-E is additive in signed main effects and has sensitivity
\(K\eta\).

The deterministic estimation condition used in Lemma~\ref{lem:fo_envelope_sensitivity} is either an explicit assumption or a high-probability event supplied by Proposition~\ref{prop:mse_uniform_event}.

For each active domain \(d\), define
\(\phi_g^+(d)=\max(\phi_g(d),0)\) and
\(\phi_g^-(d)=\max(-\phi_g(d),0)\). The raw HARP-C and HARP-E envelopes are
\begin{equation*}
q_d^{\mathrm{C}}(S)
=
b_d
+
\max_{g\in S}\phi_g^+(d)
-
\sum_{g\in S}\phi_g^-(d),
\end{equation*}
with the convention \(\max_{g\in\varnothing}\phi_g^+(d)=0\), and
\begin{equation*}
q_d^{\mathrm{E}}(S)
=
b_d
+
\sum_{g\in S}\phi_g^+(d)
-
\sum_{g\in S}\phi_g^-(d).
\end{equation*}
Equivalently, since \(\phi_g^+(d)-\phi_g^-(d)=\phi_g(d)\),

\begin{equation*}
q_d^{\mathrm{E}}(S)
=
b_d+
\sum_{g\in S}\phi_g(d).
\end{equation*}

For \(X\in\{\mathrm{C},\mathrm{E}\}\), define the clipped domain utility
\(y_d^X(S)=\mathrm{clip}_{[0,1]}(q_d^X(S))\) and the aggregate first-order
utility
\begin{equation*}
U_X(S)
=
\sum_{d\in\mathcal{A}_{\mathrm{dom}}}w_d y_d^X(S).
\end{equation*}
The estimated utility \(\widehat U_X(S)\) is defined analogously by replacing
\(\phi_g(d)\) with \(\widehat{\phi}_g(d)\). In particular,

\begin{equation*}
\widehat q_d^{\mathrm{E}}(S)
=
b_d+
\sum_{g\in S}\widehat\phi_g(d).
\end{equation*}

\begin{lemma}[Envelope clipping stability]
\label{lem:fo_envelope_clipping}
Fix \(X\in\{\mathrm{C},\mathrm{E}\}\). For each active domain \(d\), suppose
the true domain utility \(u_d(S)\in[0,1]\) satisfies
\begin{equation*}
\left|u_d(S)-q_d^X(S)\right|
\le
\xi_d^X(S)
\end{equation*}
for every feasible set \(S\). Let
\({U_{\mathcal{A}}(S)}=\sum_{d\in\mathcal{A}_{\mathrm{dom}}}w_d u_d(S)\), using the same
active-domain set and weights as \(U_X\). Then
\begin{equation*}
\left|u_d(S)-y_d^X(S)\right|
\le
\xi_d^X(S),
\end{equation*}
and therefore, for nonnegative active-domain weights,
\begin{equation*}
\left|{U_{\mathcal{A}}(S)}-U_X(S)\right|
\le
\sum_{d\in\mathcal{A}_{\mathrm{dom}}}w_d\xi_d^X(S).
\end{equation*}
If the paper-level utility \(U(S)\) is evaluated on the full downstream suite rather than only the active domains, then the difference between \(U(S)\) and \(U_{\mathcal{A}}(S)\) is absorbed into the approximation term \(\epsilon_X\) in Theorem~\ref{thm:shared_fo_stability}.
\end{lemma}

\begin{proof}
Fix \(d\), \(S\), and \(X\in\{\mathrm{C},\mathrm{E}\}\). Since
\(u_d(S)\in[0,1]\), we have
\(u_d(S)=\mathrm{clip}_{[0,1]}(u_d(S))\). The clipping map is
\(1\)-Lipschitz, so
\begin{equation*}
\begin{aligned}
& \left|u_d(S)-y_d^X(S)\right|
 \\ &=
\left|
\mathrm{clip}_{[0,1]}(u_d(S))
-
\mathrm{clip}_{[0,1]}(q_d^X(S))
\right|  \\
&\le
\left|u_d(S)-q_d^X(S)\right|
\le
\xi_d^X(S).
\end{aligned}
\end{equation*}
Multiplying by \(w_d\), summing over active domains, and applying the triangle
inequality gives
\begin{equation*}
\left|{U_{\mathcal{A}}(S)}-U_X(S)\right|
\le
\sum_{d\in\mathcal{A}_{\mathrm{dom}}}w_d\xi_d^X(S).
\end{equation*}
\end{proof}

\begin{theorem}[Shared first-order stability template]
\label{thm:shared_fo_stability}
Fix \(X\in\{\mathrm{C},\mathrm{E}\}\). Let \(U_X\) be the first-order utility
defined with the true main effects \(\phi_g(d)\), and let \(\widehat U_X\) be
its estimated version defined with \(\widehat{\phi}_g(d)\). Let
\(S^\star\in\arg\max_{c(S)\le B}U(S)\), where the maximum is taken over
feasible leaf sets, and let \(\widehat S_X\) be the set returned by the
corresponding HARP variant. Suppose that, for every feasible set \(S\),
\begin{equation*}
|U(S)-U_X(S)|\le \epsilon_X
\end{equation*}
and
\begin{equation*}
|\widehat U_X(S)-U_X(S)|\le \delta_X .
\end{equation*}
If the greedy procedure returns an \(\epsilon_{\mathrm{opt}}\)-approximate
maximizer of \(\widehat U_X\) under the budget, where \(\epsilon_{\mathrm{opt}}\)
denotes its optimization gap, then
\begin{equation*}
U(S^\star)-U(\widehat S_X)
\le
2\epsilon_X+2\delta_X+\epsilon_{\mathrm{opt}}.
\end{equation*}
\end{theorem}

\begin{proof}
By the approximation assumption,
\begin{equation*}
U(S^\star)
\le
U_X(S^\star)+\epsilon_X .
\end{equation*}
By the estimation assumption,
\begin{equation*}
U_X(S^\star)
\le
\widehat U_X(S^\star)+\delta_X .
\end{equation*}
Since \(\widehat S_X\) is an \(\epsilon_{\mathrm{opt}}\)-approximate maximizer
of \(\widehat U_X\),
\begin{equation*}
\widehat U_X(S^\star)
\le
\widehat U_X(\widehat S_X)
+
\epsilon_{\mathrm{opt}} .
\end{equation*}
Applying the estimation and approximation assumptions again,
\begin{equation*}
\begin{aligned}
\widehat U_X(\widehat S_X)
&\le
U_X(\widehat S_X)+\delta_X  \\
&\le
U(\widehat S_X)+\delta_X+\epsilon_X .
\end{aligned}
\end{equation*}
Combining the four inequalities gives
\begin{equation*}
U(S^\star)-U(\widehat S_X)
\le
2\epsilon_X+2\delta_X+\epsilon_{\mathrm{opt}}.
\end{equation*}
\end{proof}

\begin{lemma}[Envelope sensitivity to main-effect estimation error]
\label{lem:fo_envelope_sensitivity}
Suppose every feasible set contains at most \(K\) leaves, the active-domain
set \(\mathcal{A}_{\mathrm{dom}}\) is fixed, the active-domain weights are nonnegative and weights sum to one, \(b_d\) and \(w_d\) are fixed, and
\(|\widehat{\phi}_g(d)-\phi_g(d)|\le \eta\) for all \(g,d\). Then
\begin{equation*}
\delta_{\mathrm{C}}\le (K+1)\eta
\qquad\text{and}\qquad
\delta_{\mathrm{E}}\le {K\eta} .
\end{equation*}
\end{lemma}

\begin{proof}
Since \(x\mapsto \max(x,0)\) and \(x\mapsto \max(-x,0)\) are both
\(1\)-Lipschitz, the assumption
\(|\widehat{\phi}_g(d)-\phi_g(d)|\le \eta\) implies
\begin{equation*}
|\widehat{\phi}_g^+(d)-\phi_g^+(d)|\le \eta
\end{equation*}
and
\begin{equation*}
|\widehat{\phi}_g^-(d)-\phi_g^-(d)|\le \eta .
\end{equation*}
First consider HARP-C. For any feasible set \(S\) with \(|S|\le K\),
\begin{equation*}
\left|
\max_{g\in S}\widehat{\phi}_g^+(d)
-
\max_{g\in S}\phi_g^+(d)
\right|
\le
\eta,
\end{equation*}
where the empty-set convention gives zero on both sides when \(S=\varnothing\).
For the negative part,
\begin{equation*}
\left|
\sum_{g\in S}\widehat{\phi}_g^-(d)
-
\sum_{g\in S}\phi_g^-(d)
\right|
\le
K\eta .
\end{equation*}
Thus the raw HARP-C domain-envelope error is at most \((K+1)\eta\). Since
clipping to \([0,1]\) is \(1\)-Lipschitz, the clipped domain-utility error is
also at most \((K+1)\eta\). Summing over active domains with nonnegative weights that sum to
one gives
\begin{equation*}
|\widehat U_{\mathrm{C}}(S)-U_{\mathrm{C}}(S)|
\le
(K+1)\eta .
\end{equation*}
Hence \(\delta_{\mathrm{C}}\le (K+1)\eta\).

Now consider HARP-E. Because

\begin{equation*}
\phi_g^+(d)-\phi_g^-(d)=\phi_g(d)
\text{ and }
\widehat\phi_g^+(d)-\widehat\phi_g^-(d)=\widehat\phi_g(d),
\end{equation*}

we have

\begin{equation*}
q_d^{\mathrm{E}}(S)=b_d+\sum_{g\in S}\phi_g(d),
\quad
\widehat q_d^{\mathrm{E}}(S)=b_d+\sum_{g\in S}\widehat\phi_g(d).
\end{equation*}

Therefore, for any feasible set \(S\) with \(|S|\le K\),

\begin{equation*}
\begin{aligned}
\left|
\widehat q_d^{\mathrm{E}}(S)
-
q_d^{\mathrm{E}}(S)
\right|
&=
\left|
\sum_{g\in S}
\left(
\widehat\phi_g(d)-\phi_g(d)
\right)
\right|
\\
&\le
\sum_{g\in S}
|\widehat\phi_g(d)-\phi_g(d)|
\\
&\le
K\eta .
\end{aligned}
\end{equation*}

Again, clipping to \([0,1]\) is \(1\)-Lipschitz, so the clipped HARP-E domain-utility error is at most \(K\eta\). Summing over active domains with nonnegative weights that sum to one gives

\begin{equation*}
|\widehat U_{\mathrm{E}}(S)-U_{\mathrm{E}}(S)|
\le
K\eta .
\end{equation*}

Thus \(\delta_{\mathrm{E}}\le K\eta\).
\end{proof}

Combining Theorem~\ref{thm:shared_fo_stability} with
Lemma~\ref{lem:fo_envelope_sensitivity} gives the stability statements for
HARP-C and HARP-E in the main text. The shared theorem explains why both
variants have the same high-level stability form. The sensitivity lemma shows that HARP-C incurs one positive-side maximum error and up to \(K\) negative-side errors, giving \((K+1)\eta\). HARP-E, although written with positive and negative parts, is additive in the signed main effects and therefore incurs only \(K\) signed-effect errors, giving \(K\eta\).
\clearpage
\section{Experimental Hyperparameters}
\label{app:hyperparameters}

We split the hyperparameters into two tables. Table~\ref{tab:hp_shared} lists the shared final-finetuning settings (LoRA configuration, optimizer, and per-pool sequence-length prefilter) and the evaluation protocol. Table~\ref{tab:hp_harp} lists the HARP-specific hyperparameters (training-leaf size bounds, representative count, shrinkage settings, and training budget). Section~\ref{sec:exp_ablation} varies the leaf-size cap \(L=C_{\max}\) and the representative cap \(R\). All measured models use the same full validation set during selection.

\begin{table}[H]
\centering
\caption{Shared final-finetuning settings and the evaluation protocol.}
\label{tab:hp_shared}
\small
\setlength{\tabcolsep}{6pt}
\renewcommand{\arraystretch}{1.15}
\begin{tabular}{@{}p{0.66\columnwidth} r@{}}
\toprule
\textbf{Hyperparameter} & \textbf{Value} \\
\midrule
\multicolumn{2}{@{}l}{\textit{LoRA fine-tuning (Stage~2)}} \\
LoRA rank $r$ & $16$ \\
LoRA $\alpha$ & $32$ \\
LoRA dropout & $0.05$ \\
LoRA target modules & all linear projections \\
Optimizer & AdamW \\
Learning rate & $2{\times}10^{-4}$ \\
Effective batch size & $16$ \\
Epochs & $3$ \\
\midrule
\multicolumn{2}{@{}l}{\textit{Prefilter: sequence-length cap per training pool}} \\
Self-Instruct (SI) & $512$ \\
WizardLM-Evol-70k (Wiz) & $1{,}024$ \\
Tulu-3-SFT-mixture (Tulu) & $2{,}048$ \\
\midrule
\multicolumn{2}{@{}l}{\textit{Per-(model, seq-len) micro batch / grad-accum}} \\
\textsc{Qwen3-4B-Base}, $\le 1024$ & $16 / 1$ \\
\textsc{Qwen3-4B-Base}, $2048$ & $8 / 2$ \\
\textsc{Llama-3.1-8B-Base} / \textsc{Qwen3-8B-Base}, $512$ & $8 / 2$ \\
\textsc{Llama-3.1-8B-Base} / \textsc{Qwen3-8B-Base}, $1024$ & $4 / 4$ \\
\textsc{Llama-3.1-8B-Base} / \textsc{Qwen3-8B-Base}, $2048$ & $2 / 8$ \\
\midrule
\multicolumn{2}{@{}l}{\textit{Evaluation}} \\
MMLU/ARC scoring & option-letter log-probability \\
GSM8K/MATH decoding & greedy (temp $0$) \\
Max new tokens, GSM8K/MATH & $512$ \\
GPU memory utilization & $0.95$ \\
\midrule
\multicolumn{2}{@{}l}{\textit{Repetition}} \\
Random seeds per cell & $\{42, 1, 2\}$ \\
\bottomrule
\end{tabular}
\end{table}

\begin{table}[H]
\centering
\caption{HARP-specific hyperparameters.}
\label{tab:hp_harp}
\small
\setlength{\tabcolsep}{6pt}
\renewcommand{\arraystretch}{1.15}
\begin{tabular}{@{}p{0.62\columnwidth} r@{}}
\toprule
\textbf{Hyperparameter} & \textbf{Value} \\
\midrule
\multicolumn{2}{@{}l}{\textit{Hierarchy (Sec.~\ref{subsec:hierarchy})}} \\
Max leaf size $C_{\max}$ & $1024$ \\
Target min leaf size $C_{\min}$ & $\lfloor L/4\rfloor=256$ \\
\midrule
\multicolumn{2}{@{}l}{\textit{Main-effect estimation (Sec.~\ref{subsec:main_effect_estimation})}} \\
Representative cap per parent $R$ & $3$ \\
EB prior variance $\tau^2(d)$ & $0.01$ \\
SE floor & $10^{-3}$ \\
\midrule
\multicolumn{2}{@{}l}{\textit{Envelope selection (Secs.~\ref{subsec:harp_fo_cons}--\ref{subsec:harp_fo_opt})}} \\
Training budget $B$ (best-prefix cap) & $10{,}000$ \\
\bottomrule
\end{tabular}
\end{table}

\subsection{Selected Subset Sizes}

Define \(n_{\mathrm{sel}}=c(\widehat S)=|A(\widehat S)|\), the number of
selected training examples.
The HARP-selected subset size $n_{\mathrm{sel}}$ is variable and is fixed by the
best-prefix cut rather than by the budget: across all $36$ cells of
Table~\ref{tab:main_results}, HARP-C selects $1{,}087$ examples on average
(median $735.5$, range $125$--$5{,}467$) and HARP-E selects $4{,}106$ (median
$2{,}752.5$). These means match Section~\ref{sec:exp_efficiency} and correspond
to $9.2\times$ and $2.4\times$ fewer examples than the $10{,}000$-example
baseline budget, respectively. Relative to the full-finetuning mean of
$81{,}674$ examples, HARP-C uses approximately $75\times$ fewer examples.
Full FT reaches $54.7$ mean accuracy, compared with $64.2$ for HARP-C and
$63.5$ for HARP-E (Table~\ref{tab:main_results}).
The cap $B$ binds in $6$ of $36$ HARP-E cells and in no HARP-C cell.
The construction target $C_{\min}$ in Table~\ref{tab:hp_harp} is not a
lower bound on the final training-example count $n_{\mathrm{sel}}$.

\section{Computational Complexity}
\label{app:complexity}

Hierarchy construction is a one-time shared cost \(C_{\mathrm{hier}}\), independent of train--evaluate iterations and excluded from the terms below.
Table~\ref{tab:complexity} summarizes the per-cell selection and
final-training costs of HARP and the baselines. Let $N$ denote the full
training set size and $B{=}10{,}000$ the selection budget used by all
budgeted baselines. Let $T_L(x,e)$ denote the finetuning cost on $x$
examples for $e$ epochs, which is linear in $x$ at fixed $e$. HARP uses one
epoch per measured leaf and three epochs for final finetuning.
Let \(T_{\mathrm{eval}}(|\mathcal{E}|)\) denote one evaluation on the
full validation set, and let \(Q=\sum_p|\mathcal{R}_p|\) be the number
of measured training leaves. HARP's dominant selection cost is
\[
\sum_p\sum_{r\in\mathcal{R}_p}T_L(|\mathcal{L}_r|,1)
+(Q+1)T_{\mathrm{eval}}(|\mathcal{E}|).
\]
The extra evaluation scores the base model. The training term is bounded
by \(QT_L(L,1)\) when every leaf has at most \(L=C_{\max}\) examples.
The
selection-time variables are: $F$ for the n-gram feature dimension in
DSIR; $K_c$ and $\bar c$ for the number and average size of clusters in
DQ and SHED; $L$
for the leaf size cap; $R$ for the number of representative leaves per
parent used in main-effect estimation; $T_{\nabla}$ for the cost of
computing gradient features for one example in LESS; and
$T_{\mathrm{PG}}$ for the cost of one policy-gradient influence estimate
in NICE.

\begin{table}[H]
\centering
\caption{Selection and final finetuning complexity per setting. The HARP row bounds representative training and validation, excluding embedding computation, hierarchy construction, and selection search.}
\label{tab:complexity}
\small
\setlength{\tabcolsep}{4pt}
\renewcommand{\arraystretch}{1.15}
\begin{tabular}{@{}l l l@{}}
\toprule
\textbf{Method} & \textbf{Selection cost} & \textbf{Final training} \\
\midrule
Random   & $O(B)$ sampling              & $T_L(B,3)$ \\
DSIR     & $O(NF)$ feature pass         & $T_L(B,3)$ \\
\midrule
Full-FT  & ---                          & $T_L(N,3)$ \\
SHED     & $O(K_c[T_L(\bar c,1)+T_{\mathrm{eval}}(|\mathcal{E}|)])$   & $T_L(B,3)$ \\
LESS     & $O(N \cdot T_{\nabla})$ gradient datastore  & $T_L(B,3)$ \\
NICE     & $O(N \cdot T_{\mathrm{PG}})$ policy-grad.\ inf. & $T_L(B,3)$ \\
\midrule
HARP-C/E & $Q T_L(L,1)+(Q+1)T_{\mathrm{eval}}(|\mathcal{E}|)$ & $T_L(n_{\mathrm{sel}},3)$ \\
\bottomrule
\end{tabular}
\end{table}

\textsc{LESS}
builds a low-dimensional gradient datastore by running a
forward--backward pass on each training example, then selects examples
by target-task gradient similarity. \textsc{NICE} extends influence
estimation to non-differentiable evaluation metrics through
policy-gradient-style estimates, requiring repeated metric evaluation
for candidate examples. Both methods scale linearly in $N$ with a large
per-example constant: \textsc{LESS} costs $O(N \cdot T_{\nabla})$, while
\textsc{NICE} costs $O(N \cdot T_{\mathrm{PG}})$.

\end{document}